\documentclass[accepted]{uai2025} 
                        
\usepackage[american]{babel}

\usepackage{natbib} 
\usepackage{mathtools} 
\usepackage{booktabs} 
\usepackage{tikz} 
\usepackage{float}
\usepackage{siunitx}
\usepackage{setspace}
\usepackage{geometry}
\usepackage{lineno}
\usepackage{listings}
\usepackage{url}
\usepackage{subcaption}
\usepackage{amsmath,amssymb}
\usepackage{amsthm}
\usepackage{graphicx}
\usepackage{color}
\usepackage{dcolumn}
\usepackage{bm}
\usepackage{xcolor}
\usepackage{minitoc}
\usepackage{algorithm}
\usepackage{algorithmic}
\usepackage{amsmath}
\usepackage{enumitem}
\usepackage{amsthm}
\usepackage{subcaption}
\usepackage{amssymb}
\usetikzlibrary{arrows.meta}
\usetikzlibrary{shapes}

\title{Relational Causal Discovery with Latent
Confounders}

\author[1]{\href{mailto:<mnegro2@uic.edu>?Subject=Relational Causal Discovery with Latent Confounders}{Matteo Negro\textsuperscript{*}}}
\author[1]{\href{mailto:<apiras2@uic.edu>?Subject=Relational Causal Discovery with Latent Confounders}{Andrea Piras\textsuperscript{*}}}
\author[2]{Ragib Ahsan}
\author[3]{David Arbour}
\author[1]{Elena Zheleva}
\affil[1]{%
    University of Illinois Chicago\\
    Chicago
}
\affil[2]{%
    Pinterest, Inc.\\
    San Francisco
}
\affil[3]{%
    Adobe Research\\
    San Francisco
}

\newcommand\independent{\protect\mathpalette{\protect\independenT}{\perp}}
\def\independenT#1#2{\mathrel{\rlap{$#1#2$}\mkern2mu{#1#2}}}

\newtheorem{definition}{Definition}

\newtheorem{proposition}{Proposition}

\newtheorem{lemma}{Lemma}

\newtheorem{theorem}{Theorem}

\DeclareMathAlphabet\mathbfcal{OMS}{cmsy}{b}{n}
  
\begin{document}
\maketitle
\renewcommand\thefootnote{\fnsymbol{footnote}}
\footnotetext[1]{These authors contributed equally to this work.}
\begin{abstract}
Estimating causal effects from real-world relational data can be challenging when the underlying causal model and potential confounders are unknown. While several causal discovery algorithms exist for learning causal models with latent confounders from data, they assume that the data is independent and identically distributed (i.i.d.) and are not well-suited for learning from relational data. Similarly, existing \textit{relational} causal discovery algorithms assume causal sufficiency, which is unrealistic for many real-world datasets.
To address this gap, we propose RelFCI, a sound and complete causal discovery algorithm for relational data with latent confounders. Our work builds upon the Fast Causal Inference (FCI) and Relational Causal Discovery (RCD) algorithms and it defines new graphical models, necessary to support causal discovery in relational domains. We also establish soundness and completeness guarantees for relational d-separation with latent confounders. We present experimental results demonstrating the effectiveness of RelFCI in identifying the correct causal structure in relational causal models with latent confounders.

\end{abstract}

\section{Introduction}
The goal of causal discovery is to reveal causal information by analyzing observational data. 
Most causal discovery algorithms assume that the data is independent and identically distributed (i.i.d.), and that the data generation is based on a directed acyclic model \citep{heinze2018causal}. 
However, many real-world data sources, including biological and social networks, do not meet the i.i.d. assumption and contain entities which interact with each other and exhibit causal dependencies among their attributes. 
To capture such dependencies and enable causal reasoning in relational data, more expressive classes of directed graphical models \citep{maier2014reasoning,lee2016,ahsan2022non} and algorithms for relational causal discovery \citep{maier2013sound,lee2016,lee2020,ahsan2023learning} have been developed over the past decade.

Existing relational causal discovery algorithms rely on the strong assumption of causal sufficiency, i.e., all common causes of observed variables have been measured and included in the data. 
However, this assumption rarely holds for real-world data where the presence of latent confounders can invalidate the causal discovery and causal effect estimation processes. This is especially true in relational domains where capturing latent confounders in causal models is key to separating homophily-based correlations from contagion \citep{shalizi-smr11,lee-jasa21}. While multiple algorithms exist for causal discovery with latent confounders in i.i.d. data (e.g., \citet{spirtes2000causation,colombo2012learning}), none address relational data. 
To facilitate more realistic causal discovery in relational domains, it is necessary to formalize latent confounders in relational causal models and lift the assumption of causal sufficiency. 

In this work, we introduce novel graphical models and a novel relational causal discovery algorithm, RelFCI, that can capture latent confounders in relational data. We build upon the representations and algorithms for Fast Causal Inference (FCI) \citep{spirtes2000causation} and Relational Causal Discovery (RCD) \citep{maier2013sound}, neither of which is sufficient on its own. FCI performs causal discovery with latent confounders but does not address relational data, whereas RCD performs relational causal discovery through relational \textit{d}-separation but assumes causal sufficiency. We introduce new relational graphical models, \textit{Latent Relational Causal Models} (LRCMs), \textit{Maximal Ancestral Abstract Ground Graphs} (MAAGGs), and \textit{Partial Ancestral Abstract Ground Graphs} (PAAGGs), and provide a set of assumptions necessary for causal discovery with latent variables on relational causal models. These models address the unique challenges of relational data, such as variable construction across relational paths and partial observation of entities. We then show that with these new models and under our specified assumptions, the rules of FCI, combined with the rules of RCD and applied to the PAAGGs, yield a sound and complete procedure for relational causal discovery. Specifically, we prove soundness and completeness guarantees of RelFCI up to a bounded hop threshold in the presence of latent variables. We demonstrate the algorithm's correctness on experimental datasets, comparing it to existing algorithms.

\section{Related Work} 
Related work falls broadly into two categories: causal discovery in the presence of latent variables and relational causal discovery. 
Several causal discovery methods support latent confounders, but only for propositional data. \citet{spirtes2000causation} introduce FCI, a generalization of PC algorithm explicitly designed for acyclic causal models with latent confounders.  
\citet{ZHANG20081873} augments FCI with an additional set of edge-orienting rules, providing completeness of the resulting algorithm. 
\citet{mooij-pmlr20} show that FCI is sound and complete for
cyclic models under $\sigma$-separation criteria.

\citet{maier2014reasoning} considered $d$-separation semantics on relational causal models, using  \textit{abstract ground graphs}, a lifted representation. 
\citet{maier2014reasoning} further provide soundness and completeness of $m$-separation, an analogue of $d$-separation on mixed graphs \citep{Richardson2002AncestralGM}, on abstract ground graphs.
\citet{maier2013sound} introduce RCD, a sound and complete algorithm for discovery on abstract ground graphs under the assumptions of $d$-faithfulness, sufficiency, and acyclicity. 
\citet{lee2016} develop a more efficient version of RCD, RCD-Light, that requires polynomial time and space to compute. Additionally, using a novel characterization of relational causal models under different path semantics, they present an alternate technique for causal discovery \citep{lee2016ch}.
Our work can be seen as an extension of these works, which relaxes causal sufficiency in order to more closely mirror real-world cases \citep{rothenhausler-nips15,strobl-springer19}.

\section{Background}
We provide an overview of relational theory, which serves as the foundation for our proposed RelFCI algorithm and its proofs of correctness. We follow the theoretical definitions provided by \cite{maier2014reasoning}. We also go over the theory underlying causal discovery using latent confounders and partial ancestral graphs for Bayesian networks as specified by \cite{spirtes2000causation} for the FCI algorithm. Finally, we provide the set of assumptions used in this work for relational causal discovery with latent variables.
Appendices \ref{rcd} and \ref{fci} contain accompanying figures that illustrate the concepts presented in this section.
\subsection{Relational Data and Relational Causal Models}
A \textit{Relational Schema} $\mathcal{S} = (\mathcal{E}, \mathcal{R}, \mathcal{A}, \textit{card})$ is a collection of a set of entity types $\mathcal{E}$; a set of relationship types $\mathcal{R}$, where $\mathcal{R}_i=\langle E_1^i,...,E_{a}^i\rangle \in \mathcal{R}$, with $E_j^i\in\mathcal{E}$ and $a$ the arity of the relation; a set of attribute classes $\mathcal{A}(I)$ for each entity or relationship and a cardinality function \textit{card}: $\mathcal{R}\times\mathcal{E} \rightarrow \{\text{ONE, MANY}\}$. As a running example, we will consider a schema with two entity types, USER (U) and POST (P), and the relationship between them, REACTS (R). USER has three attributes (U.Type, U.Activity and U.Sentiment), POST has two attributes (P.Engagement, P.Content), and REACTS has one attribute (R.Frequency). 

Given a relational schema, a \textit{Relational Variable} $[I_X...I_Y].Y$ consists of a \textit{Relational Path} $[I_X...I_Y]$, an alternating sequence of connected entities and relations, and an attribute $Y$ of the last class reached by said path. The first class $I_X$ of this relational variable is called \textit{perspective}. For example, from the described schema example, $[U, R, P].Engagement$ is a relational variable from the perspective USER, which captures the set of Engagements of all posts that a user reacts to. 

A \textit{Relational Dependency} $[I_X...I_Y].Y \rightarrow [I_X].X$ is a pair of two relational variables with a common perspective. Relational paths allow us to model causal dependencies between attributes of different entities, e.g., $[P, R, U].Sentiment \rightarrow [P].Engagement$ indicates that the engagement of a post depends on the sentiment of the user reacting to that post. The dependency is called \textit{canonical} if the path of the outcome variable (in the example, $[P].Engagement$) has a path of length 1. 
A \textit{Relational Causal Model} $\mathcal{M}_{\Theta}(\mathcal{S}, \mathcal{D})$ is a set of relational dependencies $\mathcal{D}$ defined over schema $\mathcal{S}$, with $\Theta$ denoting the set of conditional probability distributions for each attribute $\mathcal{A}(I)$ of every class $I\in\mathcal{E}\cup\mathcal{R}$ over its parents. The arrow corresponds to a relational dependency. The example relational causal model in Figure \ref{fig:model} shows that the user's sentiment and the post's content influence the engagement of the post.
A \textit{Relational Skeleton} $\sigma$, is an instantiation of the schema for all entities, relationships, and attributes which follows the cardinality requirements specified by \textit{card}. 
In other words, this is the data realization of the schema. For example, one relational skeleton could have two entities of type USER, Bob and Anna, and one entity of type POST, a food recipe, that Bob and Anna react to.
We denote the set of all possible relational skeletons for a schema $\mathcal{S}$ as $\Sigma_{\mathcal{S}}$.

Each relational causal model $\mathcal{M}_\Theta$ and relational skeleton $\sigma$ correspond to a \textit{Ground Graph} $GG_{\mathcal{M}\sigma}$.  The nodes in this graph are the attributes of all Entities and Relation instances in the skeleton $\sigma$, while the edges between instances of variables represent all dependencies in $\mathcal{M}$. Graphical examples of these representations can be seen in Appendix \ref{rcd}. An \textit{Abstract Ground Graph} $AGG_{\mathcal{M}\mathcal{B}h}$, for the relational causal model $\mathcal{M}$, perspective $\mathcal{B}$ and hop-threshold $h$, is a graph that captures dependencies between relational variables that hold for all possible ground graphs $GG_{\mathcal{M}\sigma}$, with $\sigma \in \Sigma_{\mathcal{S}}$. Abstract ground graphs are defined for each perspective $\mathcal{B}$ and relational path length fixed to $h$. 


$AGG_{\mathcal{M}\mathcal{B}h}$ contains edges between relational variables if the instantiations of those relational variables contain a dependent pair in some ground graph. 
The edges are obtained using the \verb|extend| method \citep{maier2014reasoning}, which constructs relational paths from the current perspective to a dependency’s source attribute. Formally, given a relational dependency $[I_Y, \dots, I_Z].Y \rightarrow [I_X].X$, and a current perspective path $[I_B, \dots, I_X]$, the method finds all valid pivot points between the reversed path $[I_X, \dots, I_B]$ and the dependency path, and concatenates them at the pivot to generate new paths of the form $[I_B, \dots, I_X, \dots, I_Z]$. The process ensures that dependencies are appropriately lifted to the abstract ground graph, regardless of original perspective. Furthermore, intersection variables inherit the edges from both the variables in the pair. 

A single dependency in $\mathcal{M}$, with the extend method, may support multiple edges in $AGG_{\mathcal{M}\mathcal{B}h}$. Additionally, a single model $\mathcal{M}$ produces multiple AGGs, one for each perspective. A more complete description of AGGs components and the extend method are provided in Appendix \ref{rcd}.
\citet{maier2014reasoning} showed that $d$-separation applied to abstract ground graphs (i.e., relational $d$-separation) allows the identification of conditional independences that hold across all ground graphs. 

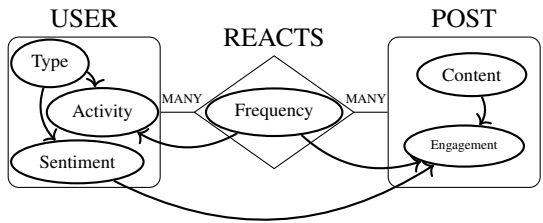
\begin{figure}[ht]
    \centering
    \begin{tikzpicture}
        \node[draw, rounded corners, rectangle, minimum width=2cm, minimum height=2cm] (A) at (0, 0) {};
        \node[above] at (A.north) {USER};
    
        \node[draw, rounded corners, rectangle, minimum width=2cm, minimum height=2cm] (B) at (5, 0) {};
        \node[above] at (B.north) {POST};
    
        \coordinate (N) at (2.5, 0.75);
        \coordinate (E) at (3.55, 0);
        \coordinate (S) at (2.5, -0.75);
        \coordinate (W) at (1.45, 0);
        \node[above] at (N.north) {REACTS};
        \node[above, font=\tiny] at ([xshift=-0.16cm] W.center) {MANY};
        \node[above, font=\tiny] at ([xshift=0.16cm] E.center) {MANY};
    
        \node[draw, ellipse, thick, font=\scriptsize] (A1) at ([yshift=-0.66cm, xshift=-0.09cm] A.center) {Sentiment};
        \node[draw, ellipse, thick, font=\scriptsize] (A2) at ([xshift=-0.45cm, yshift=0.65cm] A.center) {Type};
        \node[draw, ellipse, thick, font=\scriptsize] (A3) at ([xshift=0.24cm, yshift=0cm] A.center) {Activity};
    
        \node[draw, ellipse, font=\scriptsize, thick] (AB1_1) at ([yshift=-0.75cm] N.center) {Frequency};
        
        \draw (N) -- (E) -- (S) -- (W) -- cycle;
    
        \node[draw, ellipse, thick, font=\tiny]  (B1) at ([xshift=0cm, yshift=-0.45cm] B.center) {Engagement};
        \node[draw, ellipse, thick, font=\scriptsize] (B2) at ([xshift=0.1cm, yshift=0.5cm] B.center) {Content};
      
        \draw (A) -- (W);
        \draw (E) -- (B);
    
        \path (A1) edge[bend right, thick, ->] (B1);
        \path (B2) edge[bend left, thick, ->] (B1);
        \path (AB1_1) edge[bend left, thick, ->] (A3);
        \path (AB1_1) edge[bend right, thick, ->] (B1);
        \path (A2) edge[bend right, thick, ->] (A1);
        \path (A2) edge[bend left, thick, ->] (A3);
    \end{tikzpicture}
    \caption{Example of a Relational Causal Model}
    \label{fig:model}
\end{figure}
\subsection{Partial Ancestral Graph}
The variables of a causal graph, $V \in G$, can be divided into three categories: observed (\textbf{O}), selection (\textbf{S}), and latent (\textbf{L}) variables, denoted as G(\textbf{O},\textbf{S},\textbf{L}). In this work, we focus on latent variables and assume there are no selection ones, i.e., $\boldsymbol{S} = \emptyset$.
We denote \textbf{Cond} as the set of conditional independence relations among variables in \textbf{O}, and define the equivalence class of graphs that meets the conditional independence \textit{O-Equiv}(\textbf{Cond}) as follows: for a graph G(\textbf{O},\textbf{L}) belonging to \textit{O-Equiv}(\textbf{Cond}), given three sets of variables \textbf{X}, \textbf{Y} and \textbf{Z}, G(\textbf{O},\textbf{L}) entails that \textbf{X} $\independent$ \textbf{Y} $\mid$ \textbf{Z} if and only if \textbf{X} $\independent$ \textbf{Y} $\mid$ \textbf{Z} $\in$ \textbf{Cond}. 

An \textit{ancestral graph} \citep{ZHANG20081873} is a causal graph that can be used to represent conditional independence and causal relations of a DAG with latent variables, using only the observed variables. A path $p$ between any two vertices $X,Y\in\textbf{O}$ is called an \textit{inducing path relative to} $\langle\textbf{L}\rangle$ if every non-endpoint vertex on $p$ is either in \textbf{L} or a collider, and every collider in $p$ is an ancestor vertex of either $X$ or $Y$. A path is called primitive when \textbf{L} is empty. A \textit{Maximal Ancestral Graph} (MAG) is an ancestral graph having no primitive inducing path between any two non-adjacent vertices.

The FCI algorithm learns a Markov equivalence class of a MAG called \textit{Partial Ancestral Graph} (PAG), with edges ends having three possible marks, ◦, -, $\textgreater$, which indicate the following relationships: [1] $A \rightarrow B$ implies that $A$ causes $B$; [2] $A \leftrightarrow B$ implies a common latent confounder between the two observed variables. 
An edge has an arrowhead $\textgreater$ or tail - mark between two variables if and only if all DAGs in \textit{O-Equiv}(\textbf{Cond}) share the same arrowhead (or tail) mark for those variables, i.e. the mark is \textit{invariant}. On the other hand, if there exist two DAGs with a different edge mark between two variables, the PAG will contain a ◦ mark, i.e. the mark is \textit{variant}. 
If every circle mark corresponds to an invariant in \textit{O-Equiv}(\textbf{Cond}), the PAG is called \textit{maximally informative} for the equivalence class. Examples of MAG and PAG are available in Appendix \ref{fci}. 

\subsection{Assumptions for Relational Causal Discovery}\label{assumptions}
In this subsection, we define and discuss some key assumptions used for causal discovery in relational data, including the maximum hop threshold, d-faithfulness, acyclicity, causal sufficiency, and absence of latent descendants for latent variables.
\begin{itemize}
    \item Maximum Hop Threshold (\(h\)): The maximum hop threshold defines the largest permissible path length (or number of relational hops) between entities in a relational causal model that will be considered when constructing causal dependencies. Setting \(h\) limits the computational complexity and ensures that the discovered relationships are both interpretable and relevant. For instance, in a social network, \(h = 2\) might capture direct friendships and friends-of-friends relationships while ignoring more distant connections.
    \item D-Faithfulness: D-faithfulness (Dependency-Faithfulness) posits that any conditional independence observed in the data is also represented in the underlying causal graph, and vice versa. This ensures that the causal relationships inferred from the data align with the observed statistical dependencies in the relational causal model.
    \item Acyclicity: Acyclicity mandates that the causal graph representing the relationships among variables and entities is a directed acyclic graph (DAG). This means there are no directed cycles in the relational causal model.
    \item 
    Absence of latent descendants for latent variables: For this work, we assume that latent variables cannot be descendants of each other, i.e. all the parents and children of a latent variable are observed. This assumption is standard in constraint-based latent variable models \citep{evans2016graphs}. \citet{spirtes1995causal} note that conclusions about the equivalence class over observed variables remain valid regardless of the causal relations among latent variables.
\end{itemize}
\section{Relational Causal Discovery with Latent Variables}
In this section, we define latent variables in relational causal models, show why existing algorithms cannot perform relational causal discovery with latent variables, define the graphical models necessary for such discovery, and propose an algorithm for it. The full proofs of all theoretical findings in this paper are available in Appendix \ref{proofs}.
\subsection{Latent Variables in Relational Causal Models}
To perform causal discovery with latent confounders, we first define them in the relational context. 
Considering the set of latent variables \textbf{L}, we need to define what constitutes a latent relational variable  $[I_X...I_Y].Y\in$ \textbf{L}. We assume that all entities in $\mathcal{E}$ and relationships in $\mathcal{R}$ are observed in the relational schema and the model and, consequently, in the variable's relational path. We define the set of latent attributes in a schema $\mathcal{S}$ as the set $\mathcal{A}_\textbf{L}$. We can then look at the definition that follows: 
\begin{definition}[Latent Relational Variable]
A relational variable $RV$: $[I_X...I_Y].Y$ is considered latent, i.e., $RV\in $ \textbf{L} if and only if its attribute class $Y\in\mathcal{A}(I_Y)$ is unobserved in the schema, meaning $Y\in\mathcal{A}_\textbf{L}$.
\end{definition}
Consequently, a relational variable $RV$: $[I_W...I_Z].Z$ is observed, i.e., $RV\in $ \textbf{O} if its attribute class $Z\in\mathcal{A}(I_Z)$ is observed, indicating that it is a member of the set of observed attributes classes in the schema, which we respectively define as $\mathcal{A}_\textbf{O}$. A model's set of relational dependencies $\mathcal{D}$ is thus divided into two groups:
\begin{enumerate}
    \item Set $\mathcal{D}_\mathbf{O}$ of observed dependencies $RV_1 \rightarrow RV_2$ defined only over observed relational variables i.e., $RV_1, RV_2\in\textbf{O}$;
    \item Set $\mathcal{D}_\mathbf{L}$ of latent dependencies $RV_1 \rightarrow RV_2$ containing at least one latent relational variable i.e., $RV1\in\textbf{L}\lor RV_2\in\textbf{L}$;
\end{enumerate} 

The modified relational causal model can now be defined as follows:
\begin{definition}[Latent Relational Causal Model (LRCM)]A relational causal model with latent variables $\mathcal{M}_{\Theta L}$ consists of two parts:
\begin{enumerate}
    \item The structure $\mathcal{M}_L=(\mathcal{S}, \mathcal{D})$: the schema $\mathcal{S}$, containing a set of latent attributes $\mathcal{A}_\textbf{L}$; the set of dependencies $\mathcal{D}=\mathcal{D}_\textbf{O}\cup\mathcal{D}_\textbf{L}$ defined over all relational variables;
    \item Parameters $\Theta$: a conditional probability distribution $P([I_j].X\mid parents([I_j].X))$ for all relational variables of the form $[I_j].X$ \citep{maier2014reasoning}.
\end{enumerate}
\end{definition}
An example LRCM can be seen in figure \ref{subfig:LRCM}.
The latent AGG is constructed from LRCM $\mathcal{M}_L$, similarly to conventional relational causal models \citep{maier2014reasoning}. The construction divides the edges of the abstract ground graph into observed and unobserved edges, based on whether the underlying dependency from which the edge is yielded belongs to $\mathcal{D}_{L}$, i.e., is unobserved.
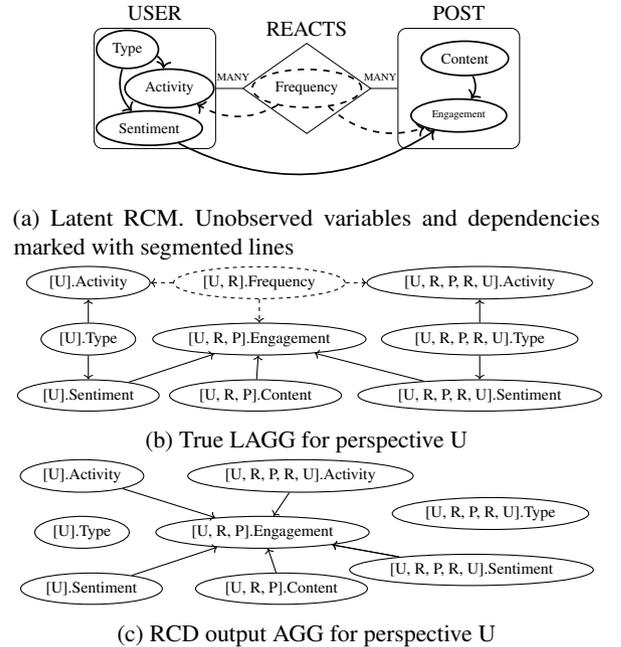
\begin{figure}[ht]
    \centering
    \begin{subfigure}{0.45\textwidth}
        \centering
        \scalebox{0.8}{
            \begin{tikzpicture}
                \node[draw, rounded corners, rectangle, minimum width=2cm, minimum height=2cm] (A) at (0, 0) {};
                \node[above] at (A.north) {USER};
            
                \node[draw, rounded corners, rectangle, minimum width=2cm, minimum height=2cm] (B) at (5, 0) {};
                \node[above] at (B.north) {POST};
            
                \coordinate (N) at (2.5, 0.75);
                \coordinate (E) at (3.55, 0);
                \coordinate (S) at (2.5, -0.75);
                \coordinate (W) at (1.45, 0);
                \node[above] at (N.north) {REACTS};
                \node[above, font=\tiny] at ([xshift=-0.16cm] W.center) {MANY};
                \node[above, font=\tiny] at ([xshift=0.16cm] E.center) {MANY};
            
                \node[draw, ellipse, thick, font=\scriptsize] (A1) at ([yshift=-0.66cm, xshift=-0.09cm] A.center) {Sentiment};
                \node[draw, ellipse, thick, font=\scriptsize] (A2) at ([xshift=-0.45cm, yshift=0.65cm] A.center) {Type};
                \node[draw, ellipse, thick, font=\scriptsize] (A3) at ([xshift=0.24cm, yshift=0cm] A.center) {Activity};
            
                \node[draw, ellipse, font=\scriptsize, thick, dashed] (AB1_1) at ([yshift=-0.75cm] N.center) {Frequency};
                
                \draw (N) -- (E) -- (S) -- (W) -- cycle;
            
                \node[draw, ellipse, thick, font=\tiny]  (B1) at ([xshift=0cm, yshift=-0.45cm] B.center) {Engagement};
                \node[draw, ellipse, thick, font=\scriptsize] (B2) at ([xshift=0.1cm, yshift=0.5cm] B.center) {Content};
              
                \draw (A) -- (W);
                \draw (E) -- (B);
            
                \path (A1) edge[bend right, thick, ->] (B1);
                \path (B2) edge[bend left, thick, ->] (B1);
                \path (AB1_1) edge[bend left, thick, ->, dashed] (A3);
                \path (AB1_1) edge[bend right, thick, ->, dashed] (B1);
                \path (A2) edge[bend right, thick, ->] (A1);
                \path (A2) edge[bend left, thick, ->] (A3);
            \end{tikzpicture}
        }
        \caption{Latent RCM. Unobserved variables and dependencies marked with segmented lines}
        \label{subfig:LRCM}
    \end{subfigure} \\ 
    \begin{subfigure}{0.45\textwidth}
        \centering
        \scalebox{0.5}{
            \begin{tikzpicture}
                \node[draw, ellipse, thick, font=\large] (1) at (-2.8, -1.5) {[U].Sentiment};
                \node[draw, ellipse, thick, font=\large] (2) at (-2.8, 0) {[U].Type};
                \node[draw, ellipse, thick, font=\large] (3) at (-2.8, 1.5) {[U].Activity};
                \node[draw, ellipse, thick, font=\large] (4) at (1.7, 0) {[U, R, P].Engagement};
                \node[draw, ellipse, thick, font=\large] (5) at (1.6, -1.5) {[U, R, P].Content};
                \node[draw, ellipse, thick, font=\large] (6) at (7.5, -1.5) {[U, R, P, R, U].Sentiment};
                \node[draw, ellipse, thick, font=\large] (7) at (7.5, 0) {[U, R, P, R, U].Type};
                \node[draw, ellipse, thick, font=\large] (8) at (7.5, 1.5) {[U, R, P, R, U].Activity};
                \node[draw, ellipse, thick, dashed, font=\large] (9) at (1.7, 1.5) {[U, R].Frequency};

                \path (2) edge[thick, ->] (1);
                \path (2) edge[thick, ->] (3);
                \path (9) edge[thick, ->, dashed] (3);
                \path (9) edge[thick, ->, dashed] (4);
                \path (1) edge[thick, ->] (4);
                \path (5) edge[thick, ->] (4);
                
                \path (7) edge[thick, ->] (6);
                \path (7) edge[thick, ->] (8);
                \path (9) edge[thick, ->, dashed] (8);
                \path (6) edge[thick, ->] (4);
            \end{tikzpicture}
        }
        \caption{True LAGG for perspective U}
        \label{subfig:True}
    \end{subfigure} \\  
    \begin{subfigure}{0.45\textwidth}
        \centering
        \scalebox{0.5}{
            \begin{tikzpicture}
                \node[draw, ellipse, thick, font=\large] (1) at (0.5, -1.5) {[U].Sentiment};
                \node[draw, ellipse, thick, font=\large] (2) at (0.25, 0) {[U].Type};
                \node[draw, ellipse, thick, font=\large] (3) at (0.25, 1.5) {[U].Activity};
                \node[draw, ellipse, thick, font=\large] (4) at (5, 0) {[U, R, P].Engagement};
                \node[draw, ellipse, thick, font=\large] (5) at (5.5, -1.5) {[U, R, P].Content};
                \node[draw, ellipse, thick, font=\large] (6) at (10.5, -1) {[U, R, P, R, U].Sentiment};
                \node[draw, ellipse, thick, font=\large] (7) at (11, 0.5) {[U, R, P, R, U].Type};
                \node[draw, ellipse, thick, font=\large] (8) at (6, 1.5) {[U, R, P, R, U].Activity};
            

                \path (1) edge[thick, ->] (4);
                \path (3) edge[thick, ->] (4);
                \path (5) edge[thick, ->] (4);
                \path (6) edge[thick, ->] (4);
                \path (8) edge[thick, ->] (4);
                \path (6) edge[thick, ->] (4);
            \end{tikzpicture}
        }
        \caption{RCD output AGG for perspective U}
        \label{subfig:RCD}
    \end{subfigure}
    \caption{Counterexample that shows RCD does not produce the correct output AGG for LRCM with a faithful oracle}
    \label{fig:RCD}
\end{figure}
\begin{definition}[Latent Abstract Ground Graph (LAGG)]
    Given a relational causal model $\mathcal{M}_L$ a maximum hop threshold $h$, and a perspective $\mathcal{B}$, the $LAGG_{\mathcal{M}_\textbf{L}\mathcal{B}h}$ is the abstract ground graph of the latent relational causal model. It contains both variables in the sets \textbf{O} and \textbf{L} over perspective $\mathcal{B}$, plus intersection variables divided into observed intersection variables (both participating variables in the intersections are observed), and latent ones (i.e., at least one participating variable is latent). The set of edges $E$ yielded from the dependencies in $\mathcal{D}_\textbf{O}$ and $\mathcal{D}_\textbf{L}$, using the \verb|extend| method \citep{maier2014reasoning}, is partitioned respectively into the set of observed ($E_\textbf{O}$) and unobserved ($E_\textbf{L}$) edges.
\end{definition} 
Consider the LRCM shown in figure \ref{subfig:LRCM} and a hop threshold $h=2$. 
The Frequency attribute for REACT is unobserved. There are six relational dependencies in the model: 1) [U].Type $\rightarrow$ [U].Activity, 2) [U].Type $\rightarrow$ [U].Sentiment, 3) [P, R, U].Sentiment $\rightarrow$ [P].Engagement, 4) [P].Content $\rightarrow$ [P].Engagement, 5) [U, R].Frequency $\rightarrow$ [U].Activity, 6) [P, R].Frequency $\rightarrow$ [P].Engagement. The last two are unobserved dependencies in $\mathcal{D}_\textbf{L}$. The respective LAGG for the described LRCM is shown in figure \ref{subfig:True}. 
\subsection{Latent Relational Causal Discovery}
The RCD algorithm is the first sound and complete procedure that learns the dependencies of a relational causal model \citep{maier2013sound}. It works under several assumptions, described in detail in Appendix \ref{assumptions}: maximum hop threshold $h$, $d$-faithfulness, acyclicity, and causal sufficiency. Causal sufficiency in particular implies that RCD was not originally designed for models with latent variables. Given that some forms of latent confounding can be detected via simple dependence tests~\citep{arbour2016inferring}, it is natural to ask whether RCD is still sound and complete when the casual sufficiency assumption is lifted. To the best of our knowledge, no prior research has addressed this question in detail. 
The following counterexample shows that RCD is neither sound nor complete for relational causal discovery with latent variables.

\textbf{Counterexample} Figure \ref{subfig:RCD} shows the output AGG produced by RCD using an oracle faithful to the underlying distribution. As we can see, the actual LAGG in figure \ref{subfig:True} contains outgoing edges from [U].Type and from [U, R, P, R, U].Type; however, the output AGG (Fig. \ref{subfig:RCD} lacks these edges. This indicates that RCD fails to identify the relational dependencies 1) and 2).
Furthermore, without latent variables, RCD cannot capture and detect the presence of a latent confounder on the AGG using only directed edges. As seen in Figure \ref{fig:RCD}, the fundamental problem that renders RCD neither sound nor complete for LRCM is the lack of identification of latent variables. This suggests that a more expressive representation than AGGs is required for the correct causal discovery in the presence of latent variables. 

\subsection{Partial Ancestral Abstract Ground Graphs}

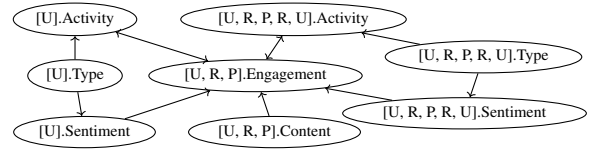
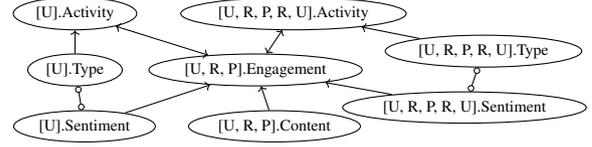
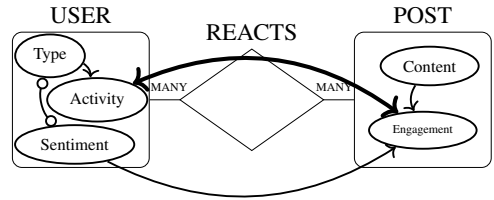
\begin{figure}[ht]
    \centering
    \begin{subfigure}{0.45\textwidth}
        \centering
        \scalebox{0.5}{
            \begin{tikzpicture}
                \node[draw, ellipse, thick, font=\large] (1) at (0.5, -1.5) {[U].Sentiment};
                \node[draw, ellipse, thick, font=\large] (2) at (0.25, 0) {[U].Type};
                \node[draw, ellipse, thick, font=\large] (3) at (0.25, 1.5) {[U].Activity};
                \node[draw, ellipse, thick, font=\large] (4) at (5, 0) {[U, R, P].Engagement};
                \node[draw, ellipse, thick, font=\large] (5) at (5.5, -1.5) {[U, R, P].Content};
                \node[draw, ellipse, thick, font=\large] (6) at (10.5, -1) {[U, R, P, R, U].Sentiment};
                \node[draw, ellipse, thick, font=\large] (7) at (11, 0.5) {[U, R, P, R, U].Type};
                \node[draw, ellipse, thick, font=\large] (8) at (6, 1.5) {[U, R, P, R, U].Activity};

                \path (2) edge[thick, ->] (3);
                \path (2) edge[thick, ->] (1);
                \path (4) edge[thick, <->] (3);
                \path (1) edge[thick, ->] (4);
                \path (5) edge[thick, ->] (4);
                \path (6) edge[thick, ->] (4);
                \path (8) edge[thick, <->] (4);
                \path (7) edge[thick, ->] (6);
                \path (7) edge[thick, ->] (8);
            \end{tikzpicture}
        }
        \caption{MAAGG from the LRCM in Figure \ref{subfig:LRCM} for perspective U}
    \end{subfigure} \\
    \begin{subfigure}{0.45\textwidth}
        \centering
        \scalebox{0.5}{
            \begin{tikzpicture}
                \node[draw, ellipse, thick, font=\large] (1) at (0.5, -1.5) {[U].Sentiment};
                \node[draw, ellipse, thick, font=\large] (2) at (0.25, 0) {[U].Type};
                \node[draw, ellipse, thick, font=\large] (3) at (0.25, 1.5) {[U].Activity};
                \node[draw, ellipse, thick, font=\large] (4) at (5, 0) {[U, R, P].Engagement};
                \node[draw, ellipse, thick, font=\large] (5) at (5.5, -1.5) {[U, R, P].Content};
                \node[draw, ellipse, thick, font=\large] (6) at (10.5, -1) {[U, R, P, R, U].Sentiment};
                \node[draw, ellipse, thick, font=\large] (7) at (11, 0.5) {[U, R, P, R, U].Type};
                \node[draw, ellipse, thick, font=\large] (8) at (6, 1.5) {[U, R, P, R, U].Activity};

                \path (2) edge[thick, ->] (3);
                \path (2) edge[thick, {Circle[open]}-{Circle[open]}] (1);
                \path (4) edge[thick, <->] (3);
                \path (1) edge[thick, ->] (4);
                \path (5) edge[thick, ->] (4);
                \path (6) edge[thick, ->] (4);
                \path (8) edge[thick, <->] (4);
                \path (7) edge[thick, {Circle[open]}-{Circle[open]}] (6);
                \path (7) edge[thick, ->] (8);
            \end{tikzpicture}
        }
        \caption{PAAGG of \textit{O-Equiv}($\mathcal{D}_{\textbf{O}}$) for perspective U}
    \end{subfigure} \\
    \begin{subfigure}{0.45\textwidth}
        \centering
        \scalebox{0.9}{
            \begin{tikzpicture}
                \node[draw, rounded corners, rectangle, minimum width=2cm, minimum height=2cm] (A) at (0, 0) {};
                \node[above] at (A.north) {USER};
            
                \node[draw, rounded corners, rectangle, minimum width=2cm, minimum height=2cm] (B) at (5, 0) {};
                \node[above] at (B.north) {POST};
            
                \coordinate (N) at (2.5, 0.75);
                \coordinate (E) at (3.55, 0);
                \coordinate (S) at (2.5, -0.75);
                \coordinate (W) at (1.45, 0);
                \node[above] at (N.north) {REACTS};
                \node[above, font=\tiny] at ([xshift=-0.16cm] W.center) {MANY};
                \node[above, font=\tiny] at ([xshift=0.16cm] E.center) {MANY};
            
                \node[draw, ellipse, thick, font=\scriptsize] (A1) at ([yshift=-0.66cm, xshift=-0.09cm] A.center) {Sentiment};
                \node[draw, ellipse, thick, font=\scriptsize] (A2) at ([xshift=-0.45cm, yshift=0.65cm] A.center) {Type};
                \node[draw, ellipse, thick, font=\scriptsize] (A3) at ([xshift=0.24cm, yshift=0cm] A.center) {Activity};
                
                \draw (N) -- (E) -- (S) -- (W) -- cycle;
            
                \node[draw, ellipse, thick, font=\tiny]  (B1) at ([xshift=0cm, yshift=-0.45cm] B.center) {Engagement};
                \node[draw, ellipse, thick, font=\scriptsize] (B2) at ([xshift=0.1cm, yshift=0.5cm] B.center) {Content};
              
                \draw (A) -- (W);
                \draw (E) -- (B);
            
                \path (A1) edge[bend right, thick, ->] (B1);
                \path (B2) edge[bend right, thick, ->] (B1);
                \path (A2) edge[bend right, thick, {Circle[open]}-{Circle[open]}] (A1);
                \path (A2) edge[bend left, thick, ->] (A3);
                \path (A3) edge[bend left, thick, <->, line width=0.6mm] (B1);
            \end{tikzpicture}
        }
        \caption{PARM of \textit{O-Equiv}($\mathcal{D}_{\textbf{O}}$)}
    \end{subfigure}%
    \caption{New representations to enable relational causal discovery in LRCMs}
    \label{fig:models}
\end{figure}

To enable causal discovery in latent relational causal models, we first need to define the necessary graphical models. Figure 3 shows an example of the models we introduce in this section. We will do so by considering an extension of the graphical models used in FCI to the relational setting, with a representation that expresses the underlying dependencies $\mathcal{D}$ over a set of observed variables \textbf{O} (i.e., $\mathcal{D}_{\textbf{O}}$):
\begin{definition}[Maximal Ancestral Abstract Ground Graph (MAAGG)]
Given a 
latent relational causal model $\mathcal{M}_\textbf{L}(\mathcal{S},\mathcal{D})$ with a hop threshold $h$, any perspective $\mathcal{B}$, and the resulting Latent Abstract Ground Graph $LAGG_{\mathcal{M}_L\mathcal{B}h}$, the maximal ancestral abstract ground graph $MAAGG_{\mathcal{M}\mathcal{B}h'}$ is a graph, with a hop threshold $h'\geq h$ comprising:
\begin{itemize}
    \item One node for each relational variable of the LAGG in \textbf{O} and the respective set of observed intersection variables;
    \item Three types of edges: $\rightarrow$, \textemdash , and $\leftrightarrow$, which are used to represent the underlying dependencies $\mathcal{D}_{\textbf{O}}$.
\end{itemize}
\end{definition}
The MAAGG $\mathbfcal{G}$ defined over the variables in \textbf{O}, following the definition of \citet{ZHANG20081873}, probabilistically represents the respective LAGG defined over \textbf{O} and \textbf{L}, specifically:
\begin{itemize}
    \item Two variables $A, B \in \textbf{O}$ are adjacent in $\mathbfcal{G}$ if and only if there is an inducing path relative to $\langle L\rangle$ in the true LAGG;
    \item The orientation entails the same concept of non-causality and ancestry between two variables for MAGs and PAGs.
\end{itemize}
We now introduce a lemma that is necessary for proving the soundness and completeness of our proposed method:
\begin{lemma}
    Given a relational causal model structure $\mathcal{M}$ and perspective $\mathcal{B}$, an abstract ground graph $AGG_{\mathcal{M}\mathcal{B}}$ is ancestral if and only if all ground graphs $GG_{\mathcal{M}\sigma}$, with skeleton $\sigma\in\sum_\mathcal{S}$, are ancestral.
\end{lemma}
Generally, the set underlying dependencies $\mathcal{D}_\textbf{O}$ is not associated with a single MAAGG, but with the class of Markov equivalence defined as \textit{O-Equiv}($\mathcal{D}_{\textbf{O}}$). Therefore, we define another abstraction, based on PAGs and AGGs, that represents this equivalence class: 
\begin{definition}[Partial Ancestral Abstract Ground Graph (PAAGG)]
Given a relational causal model $\mathcal{M}_\textbf{L}(\mathcal{S},\mathcal{D})$ with hop threshold $h$, the respective $MAAGG_{\mathcal{M}\mathcal{B}_{h'}}$, and its equivalence class \textit{O-Equiv}($\mathcal{D}_{\textbf{O}}$), and a perspective $\mathcal{B}$, the partial ancestral abstract ground graph $PAAGG_{\mathbfcal{M}\mathcal{B}_{h'}}$ is a bidirected PAG, with hop threshold $h'\geq h$ comprising:
\begin{itemize}[wide, labelwidth=!, labelindent=0pt]
    \item The same set of nodes and adjacencies as the MAAGG;
    \item Edges containing three kinds of marks: ◦, \textemdash and $\rightarrow$, which are used to represent the variance and invariance of the equivalence class \textit{O-Equiv}($\mathcal{D}_{\textbf{O}}$).
\end{itemize}
\end{definition}



The following proposition provides a description of the soundness and completeness of the new representation:
\begin{proposition}
    Given a relational causal model $\mathcal{M}_\textbf{L}(\mathcal{S},\mathcal{D})$ with hop threshold $h$, and its respective latent abstract ground graph $G$, the constructed MAAGG probabilistically and causally represents $G$ and thus the underlying relational causal model. Furthermore, assuming a sound and complete procedure to construct the PAAGG $G'$, it correctly represents the Markov equivalence class of the produced MAAGG
    and, therefore, of $G$ and the underlying model $\mathcal{M}_\textbf{L}$. 
\end{proposition}
The equivalence class of MAAGGs represented by the PAAGG corresponds to multiple LRCMs that share the same set of dependencies over $\mathcal{D}_\textbf{O}$. Thus, it is possible to define a new model from the PAAGG, which represents the equivalence class \textit{O-Equiv}($\mathcal{D}_{\textbf{O}}$):


\begin{definition}[Partial Ancestral Relational Model (PARM)]
Given a LRCM $\mathcal{M}_\textbf{L}(\mathcal{S},\mathcal{D})$ and its respective $PAAGG_{\mathbfcal{M}\mathcal{B}_{h'}}$ for the equivalence class \textit{O-Equiv}($\mathcal{D}_{\textbf{O}}$), a partial ancestral model $\mathbfcal{M}(\mathcal{S}_\textbf{O},\mathcal{D}')$ is the relational causal model abstracted by the PAAGG that represents \textit{O-Equiv}($\mathcal{D}_{\textbf{O}}$). The PARM is defined over a relational schema containing only observed attribute classes $\mathcal{S}_\textbf{O}=(\mathcal{E}, \mathcal{R}, \mathcal{A}_\textbf{O}, \textit{card})$ and a set $\mathcal{D}'$ of dependencies, which are used to represent the causality information for all models in \textit{O-Equiv}($\mathcal{D}_{\textbf{O}}$).
\end{definition}
The definitions of $MAAGG$ and $PAAGG$ allow a limit on the hop threshold higher than that of the underlying equivalence class of models. This is because the set of possible underlying dependencies with at most the same hop threshold would not capture paths that, in addition to the allowed threshold for observed variables, include unobserved variables. 
The higher hop threshold implied by definition 2 for $PAAGG$s is required to obtain an abstraction that correctly represents the presence of latent confounders in the underlying model. This is due to the presence of latent confounders (Definition 3) in a relational causal model and to the absence of latent parents and children for latent variables.
\begin{proposition} \label{prop:hop}
Given a latent relational causal model $\mathcal{M}_\textbf{L}(\mathcal{S},\mathcal{D})$ with hop threshold $h$ and its corresponding PARM $\mathbfcal{M}$, the hop threshold $h'$ of the $PAAGG_{\mathbfcal{M}\mathcal{B}}$ for any perspective $\mathcal{B}$ can be at most $2h$.
\end{proposition}
 
\subsection{The RelFCI Algorithm}
In this section, we present the Relational Fast Causal Inference (RelFCI) algorithm, a sound and complete procedure for determining causal relationships from relational data when unobserved variables are present. 
RelFCI follows a three-step approach similar to the FCI algorithm for Bayesian networks (Spirtes 2013). RelFCI adapts the FCI procedure to relational causal models, similar to how RCD \citep{maier2013sound} does with the PC algorithm \citep{spirtes2000causation}. Given that FCI is an extension of PC, RelFCI follows the same orientation rules as RCD and also assumes a prior relational skeleton. However, it differs from RCD in two ways: [1] RelFCI uses partial ancestral abstract ground graphs, one for each perspective, as the underlying representation; [2] RelFCI applies seven additional rules from FCI to ensure soundness and completeness with latent relational data.

Algorithm \ref{alg:main_alg} shows the high-level pseudocode for RelFCI, and Appendix \ref{algo} contains the complete algorithm pseudocode. RelFCI, like RCD, computes the set of potential dependencies in canonical form, limited by the $h'=2h$ threshold. Starting from these dependencies, the algorithm constructs PAAGGs, one for each perspective, all with ◦ edge marks. The first step is to remove potential dependencies using conditional independence tests with conditioning sets of increasing size drawn from the collection of adjacencies of the two nodes considered. After deleting all possible edges, a set of unshielded triples is obtained. The second phase detects colliders while finding potential additional independence relationships between the triples' variables and potentially eliminating the respective edges. Even though RelFCI operates on different graphical models compared to FCI and RCD, it is straightforward to adapt their rules for RelFCI. The third step thus performs edge orientation by applying RBO, KNC, CA, and MK3 rules from RCD first, then rules R4 through R10 from FCI. A detailed description of these rules is provided in Appendix \ref{rules}. In contrast to the first step, the latter two differ from FCI because they apply the RBO rule from RCD and propagate each edge orientation to other PAAGGs. All steps are performed on all PAAGGs to accurately identify additional separation sets for each perspective. 

Before demonstrating the soundness and completeness of RelFCI, we first clarify how the algorithm handles relational dependencies and edge orientations in PAAGGs. 
Since with ◦ marks RelFCI produces an equivalence class rather than a single causal model, certain underlying dependencies remain ambiguous. To address this, we distinguish between \textit{required dependencies}, which must be oriented in a specific direction to respect the PAAGG orientation, and \textit{possible dependencies}, which may have alternative orientations while remaining consistent with the learned PAAGG. With this new distinction, it is then possible to define the propagation of edges orientation across all PAAGGs for every perspective in a given LRCM, following a similar approach to the one described in RCD \cite{maier2013sound} for regular AGGs. A detailed explanation of these aspects is provided in Appendices \ref{deps} and \ref{edges}.

\begin{algorithm}[tb]
\caption{RelFCI algorithm}
\label{alg:main_alg}
\textbf{Input}: schema, oracle, threshold \\
\textbf{Output}: Dependencies
\begin{algorithmic}[1] 
\STATE $PDs \gets$ get potential dependencies from the base schema with 2*threshold
\STATE $PAAGGs \gets$ construct PAAGGs from set of potential dependencies $PDs$
\STATE $S \gets \{\}$ \\
\STATE $PAAGGs, S, U \gets$ find all independent variables in the graphs, storing separating sets and unshielded triples
\STATE $PAAGGS, S \gets$ orient v-structures using CD, starting from unshielded triples in $U$
\STATE $PAAGGs, S \gets$ orient PAAGGs edges using RCD and FCI rules
\STATE $Deps \gets$ retrieve underlying dependencies from the edges of oriented PAAGGs
\STATE \textbf{return} Deps
\end{algorithmic}
\end{algorithm}

\subsection{Soundness}
\citet{maier2013sound} prove the soundness of CD, KNC, CA, MR3, and the new RBO rule using a proof derived from the soundness definition presented in \citet{meek1995causal}. Thus, we will focus on the soundness of the remaining rules R4-R10 adapted from FCI \citep{ZHANG20081873}.
\begin{theorem} \label{theo:sound}
Let G be the partially oriented PAAGG from perspective B, with the correct adjacencies, unshielded colliders correctly orientated through CD and RBO, and as many edges as possible oriented through KNC, CA, MR3, and the purely common cause of RBO. Then, FCI's rules R4-R10 and the orientation propagations are sound.
\end{theorem}
The proof is an extension of those presented by \citet{spirtes2000causation} for rule R4 and \citet{ZHANG20081873} for rules R5-R10.
\subsection{Completeness}
A set of orientation rules is called complete if it generates a maximally informative graph. In PAAGGs, each circle corresponds to a variation mark in the equivalence class \textit{O-Equiv}($\mathcal{D}_{\textbf{O}}$) (modified from \citet{ZHANG20081873}). The rules employed in FCI can be divided into two groups based on their function: those used to identify arrowhead invariants (CD, KNC, CA, MR3, and R4) and those used to identify tail invariants (R5-R10). According to \citet{ali2012towards}, the first set of rules covers all invariant arrowheads. Lemma \ref{lemma:arrow} shows that PAAGGs have similar arrowhead completeness, which can be used to prove overall rule completeness.
\begin{lemma} \label{lemma:arrow}
Let G be a partially oriented PAAGG with correct adjacencies. Then, exhaustively applying CD, RBO, KNC, CA, and MR3, all with orientation propagation of edges, produces a graph G' in which for every circle mark, there exists a MAAGG in the \textit{O-Equiv}($\mathcal{D}_{\textbf{O}}$) class with a corresponding tail mark.
\end{lemma}
Following \citet{ali2012towards}'s proofs for MAG, we apply the same reasoning for MAAGG and expand it with the RBO rule. The orientation propagation proof is identical to the one offered in \citet{maier2013sound}.
We now provide tail completeness of the remaining set of rules.
\begin{lemma} \label{lemma:tail}
Let G' be the partially oriented PAAGG with correct adjacencies and unshielded colliders, and as many edges orientated with KNC, CA, and MR3, consistently applying edge propagation. Then, applying rules R5-R10, along with orientation propagation, provides a graph G'' such that for every circle mark, there exists a MAAGG in \textit{O-Equiv}($\mathcal{D}_{\textbf{O}}$) in which the associated mark is an arrowhead.
\end{lemma}

\begin{figure*}[ht]
    \centering
    \includegraphics[width=\textwidth]{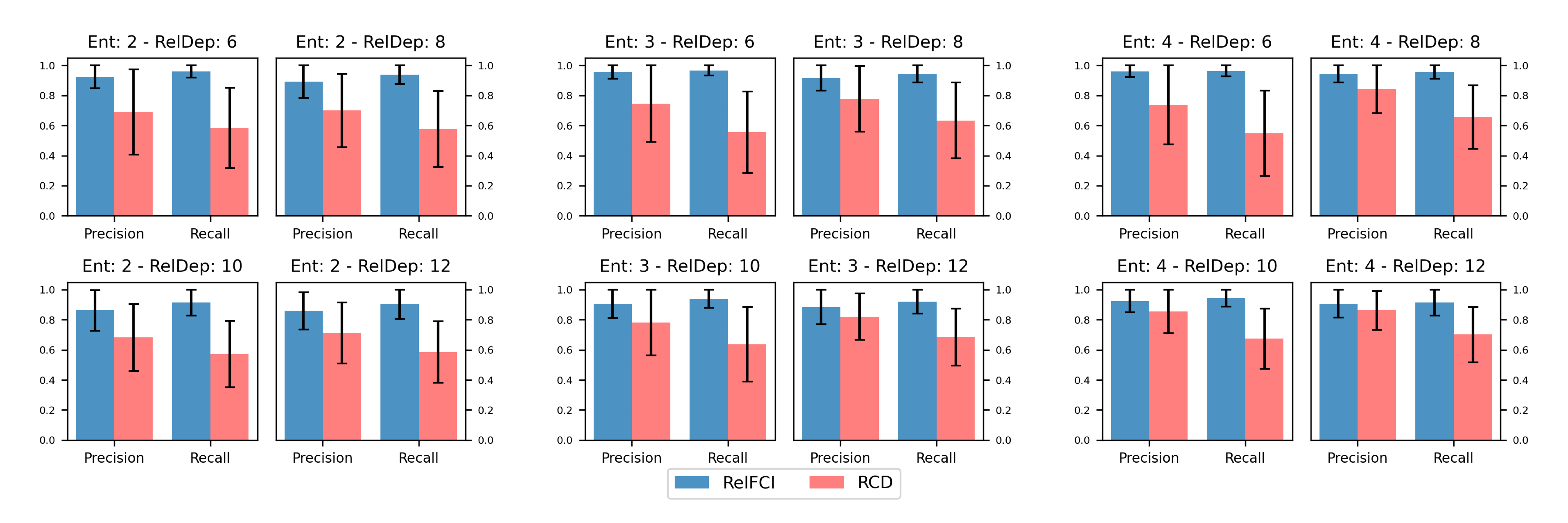}
    \caption{RelFCI and RCD Precision and Recall comparison. Results are combined for both 1 and 2 latent variables. Intervals represent $\pm1$ standard deviation.}
    \label{fig:pre-rec-tot}
\end{figure*}
\begin{figure*}[ht]
    \centering
    \includegraphics[width=\textwidth]{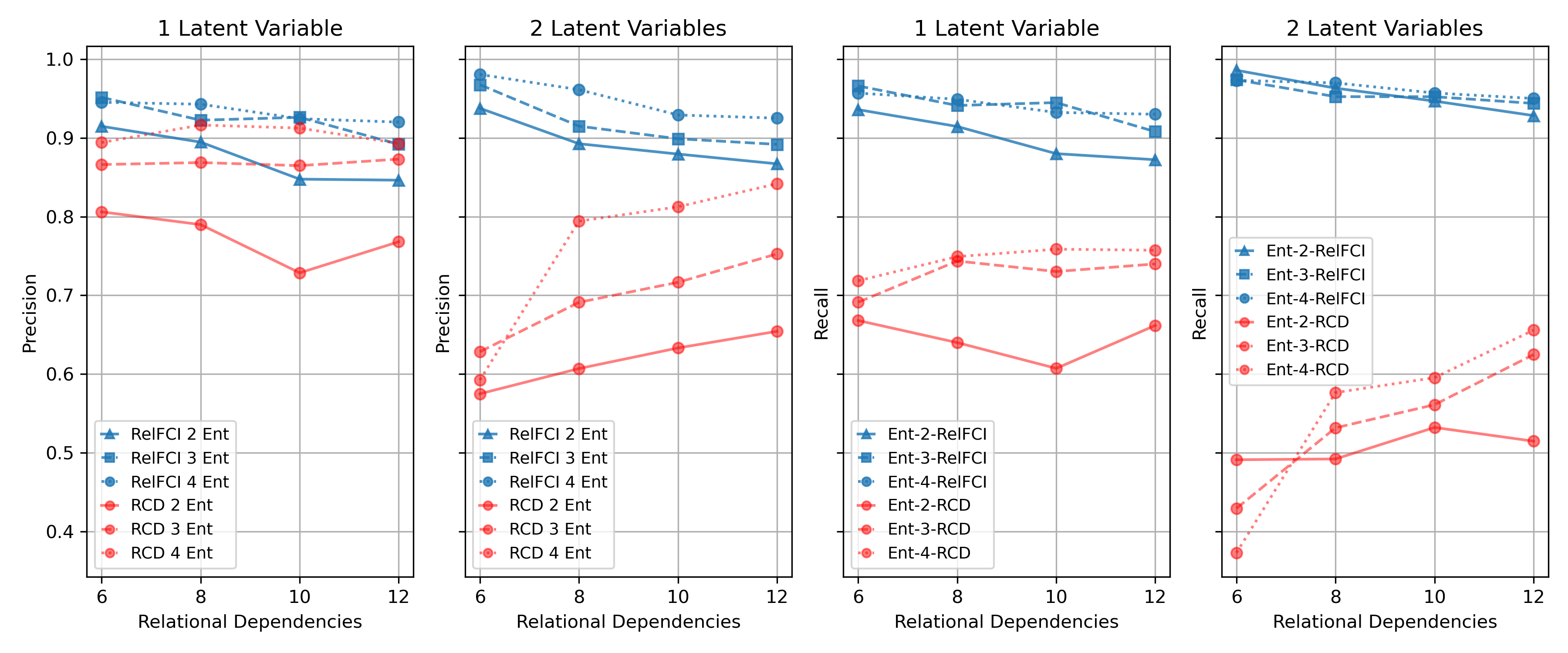}
    \caption{RelFCI and RCD Precision and Recall performance with 1 and 2 latent variables.}
    \label{fig:pre-rec-lat}
\end{figure*}

The proof comes from \cite{ZHANG20081873} tail completeness, establishing that every PAAGG edge ◦\textemdash, ◦\textemdash◦, ◦$\rightarrow$, the circle mark corresponds to an arrowhead in an MAAGG belonging to the equivalence class. With lemmas \ref{lemma:arrow} and \ref{lemma:tail} in place, completeness follows:

\begin{theorem} \label{theo:compl}
    Given a partially oriented PAAGG G with the appropriate set of adjacencies, applying rules CD, KNC, CA, MR3, and RBO extensively, followed by orienting any possible edges with rules R4-R10, all with orientation propagation, yields a maximally informative graph G. 
\end{theorem}
\begin{proof}
    Lemmas \ref{lemma:arrow} and \ref{lemma:tail} prove that in the output graph $G$ produced by applying all the rules, every remaining circle mark corresponds to both tail and arrowhead variant marks in the \textit{O-Equiv} ($\mathcal{D}_{\textbf{O}}$). As such, the circle mark is considered a variation mark. Thus, by definition, the graph $G$ is maximally informative.
\end{proof}
 
We are now ready to establish the soundness and completeness of RelFCI:
\begin{theorem}
Given a schema and a probability distribution P(\textbf{V}) with $\textbf{V}=\textbf{O}\cup\textbf{L}\cup\textbf{S}$, the output of RelFCI is a correct maximally informative PAAGG, and thus a maximally informative PARM $\mathbfcal{M}$, assuming perfect conditional independence tests and sufficient hop threshold $h'$.
\end{theorem}

\section{Experimental Results} 
\subsection{Setup}\label{sec:setup}
Our RelFCI algorithm implementation\footnote{Code available at \url{https://github.com/edgeslab/RelFCI}.}
is based on the RFCI algorithm \cite{colombo2012learning} rather than FCI. 

RFCI performs significantly fewer conditional independence tests than other FCI variants. While not proven complete, experiments show it achieves similar accuracy in edge orientation.
We generate synthetic data using a procedure similar to \cite{maier2013sound} but with the addition of introducing latent variables into the schema and model. We generate 1000 random LRCMs from randomly generated schemas for each of the following combinations: number of entities $n\in [2,4]$; $n-1$ relationships with randomly selected cardinalities; attributes per item drawn from a Poisson distribution Pois($\lambda=1$) + 1; and the number of relational dependencies (6, 8, 10, and 12) limited by hop threshold of 2. We additionally require the presence of one or two latent attributes, which are randomly chosen from the set of attributes for relational variables in the LAGG involved in at least two dependencies as the cause variable. The process yields a total of 22,000 synthetic models. We use an oracle to perform conditional independent tests for RelFCI and RCD for all possible perspectives. The results are then averaged over multiple runs for every combination, i.e., averaging over 1000 different LRCMs sharing the same properties.

\subsection{Evaluation}
We evaluated our work by comparing the model derived from the algorithm's dependencies to the ground truth.
We define the latent relational causal model obtained as ground truth by replacing the latent variable with double arrowhead edges using the same Maximal Ancestral Graph construction approach as presented in \citet{ZHANG20081873}. 
We label a missing edge as a false negative, an additional edge as a false positive, and a correct edge as a true positive and compute the precision and recall. Furthermore, to assess the necessity of new rules for relational causal discovery, we also measure the frequency with which each rule was invoked during the RelFCI runs. This last result can be found in Appendix \ref{res}.
\subsection{Results}
Figure \ref{fig:pre-rec-tot} presents a comparative analysis of RelFCI and RCD regarding precision and recall. An apparent discrepancy can be noticed in the results. This difference arises due to latent variables, which RCD fails to handle effectively. As previously discussed, the influence of hidden confounders violates RCD's core assumptions, significantly degrading its accuracy. 
In contrast, our proposed method, RelFCI, is designed to be sound and complete in the presence of latent variables. Since the RFCI implementation can sometimes introduce spurious edges or omit true ones, we expect its precision and recall to be slightly below one, as supported by \citet{colombo2012learning}. Furthermore, RelFCI exhibits a smaller variance than RCD. This indicates that RelFCI produces more consistent and reliable results across different conditions, reinforcing its robustness in handling latent variables. 

Figure \ref{fig:pre-rec-lat} further illustrates the performance trends with either one or two latent variables as the number of entities and dependencies increases. A key observation is that while RCD's performances slightly improve as the number of entities and dependencies grows, its precision and recall remain consistently lower than those of RelFCI. This trend is particularly noticeable in recall, suggesting that RCD benefits marginally from increased structural complexity. RelFCI, instead, maintains stable and high precision and recall across all conditions. These findings highlight the robustness of RelFCI in handling relational datasets with latent variables, where RCD struggles to achieve comparable accuracy.

As an additional analysis, we evaluated our new algorithm's rule activation distribution over all synthetic runs. Rules unique to FCI account for approximately one-third of all orientations. It demonstrates that latent confounders impact the entire model structure during the learning process. The plot of the rules distribution is shown in Figure \ref{fig:rule-distr}.

\begin{figure}[ht]
    \centering
    \includegraphics[width=0.45\textwidth]{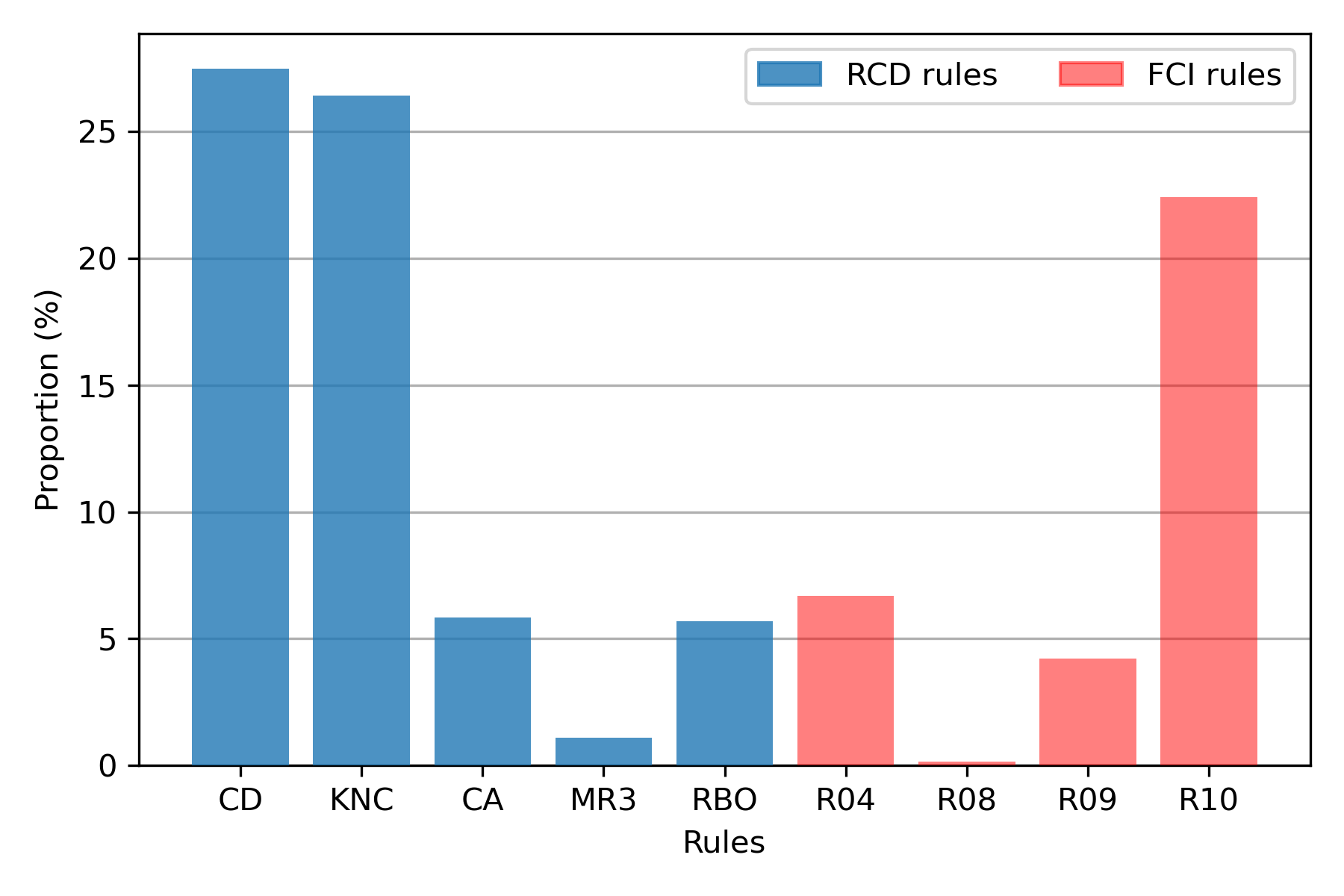}
    \caption{RelFCI's rule distribution of RCD and FCI rules.}
    \label{fig:rule-distr}
\end{figure}

\section{Conclusion}
In this paper, we provide novel representations for relational causal models with latent confounders.  
We present a sound and complete algorithm, RelFCI, for detecting causality relationships from relational data with latent confounders, which provides a more comprehensive understanding of relational causal models. 
To the best of our knowledge, this approach is the first to study relational causal discovery with latent variables. 
We believe this work will be critical in enabling causal effect estimation in complex relational systems for which the underlying causal model is unknown. 
Areas of future work include investigating the effects of including selection bias and cycles into latent relational causal models.

\begin{acknowledgements}
This research was funded in part by NSF under grant no. 2047899. The authors would like to thank Aleksandr Elifirenko for valuable analysis on the code. 
\end{acknowledgements}

\bibliography{uai2025-template}

\doublespacing

\renewcommand{\thefigure}{\arabic{figure}} 
\renewcommand{\thetable}{S\arabic{table}} 
\setcounter{lemma}{0}

\definecolor{background-color}{gray}{0.98}
\definecolor{backcolour}{rgb}{0.95,0.95,0.92}
\definecolor{codegreen}{rgb}{0,0.6,0}

\lstdefinestyle{myStyle}{
    backgroundcolor=\color{backcolour},   
    commentstyle=\color{codegreen},
    basicstyle=\ttfamily\footnotesize,
    breakatwhitespace=false,         
    breaklines=true,                 
    keepspaces=true,                 
    numbers=left,       
    numbersep=5pt,                  
    showspaces=false,                
    showstringspaces=false,
    showtabs=false,                  
    tabsize=2,
}

\lstset{style=myStyle}

\newcommand\circleast{\mathrel{{\circ}\!{-}\!{\ast}}}
\newcommand\leftarrowast{\mathrel{{\leftarrow}\!{\ast}}}
\newcommand\rightarrowast{\mathrel{{\ast}\!{\rightarrow}}}
\newcommand\rightarrowcircle{\mathrel{{\circ}\!{\rightarrow}}}
\newcommand\leftarrowcircle{\mathrel{{\leftarrow}\!{\circ}}}

\newpage

\onecolumn

\title{Relational Causal Discovery with Latent Confounders\\(Supplementary Material)}
\maketitle

\begin{appendix}
\section{Background}
\subsection{Relational Data}\label{rcd}
In this subsection, we provide possible examples of relational data. Figure \ref{fig:schema} shows an example relational schema with two entities, USER (E) and POST (P), and the relationship between them, REACTS (P), with a MANY TO MANY cardinality, meaning users can react to multiple posts and vice versa. The USER type has three attributes: Type, Sentiment, and Activity, while the POST entity type has the attributes Content and Engagement. The relationship type REACTS instead has the attribute Frequency. 
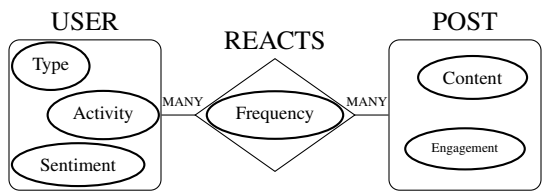
\begin{figure}[ht]
    \centering
    \begin{tikzpicture}
        \node[draw, rounded corners, rectangle, minimum width=2cm, minimum height=2cm] (A) at (0, 0) {};
        \node[above] at (A.north) {USER};
    
        \node[draw, rounded corners, rectangle, minimum width=2cm, minimum height=2cm] (B) at (5, 0) {};
        \node[above] at (B.north) {POST};
    
                \coordinate (N) at (2.5, 0.75);
                \coordinate (E) at (3.55, 0);
                \coordinate (S) at (2.5, -0.75);
                \coordinate (W) at (1.45, 0);
                \node[above] at (N.north) {REACTS};
                \node[above, font=\tiny] at ([xshift=-0.16cm] W.center) {MANY};
                \node[above, font=\tiny] at ([xshift=0.16cm] E.center) {MANY};
            
                \node[draw, ellipse, thick, font=\scriptsize] (A1) at ([yshift=-0.66cm, xshift=-0.09cm] A.center) {Sentiment};
                \node[draw, ellipse, thick, font=\scriptsize] (A2) at ([xshift=-0.45cm, yshift=0.65cm] A.center) {Type};
                \node[draw, ellipse, thick, font=\scriptsize] (A3) at ([xshift=0.24cm, yshift=0cm] A.center) {Activity};
            
                \node[draw, ellipse, font=\scriptsize, thick] (AB1_1) at ([yshift=-0.75cm] N.center) {Frequency};
                
                \draw (N) -- (E) -- (S) -- (W) -- cycle;
            
                \node[draw, ellipse, thick, font=\tiny]  (B1) at ([xshift=0cm, yshift=-0.45cm] B.center) {Engagement};
                \node[draw, ellipse, thick, font=\scriptsize] (B2) at ([xshift=0.1cm, yshift=0.5cm] B.center) {Content};
              
                \draw (A) -- (W);
                \draw (E) -- (B);
    
    \end{tikzpicture}
    \caption{Example of Relational Schema}
    \label{fig:schema}
\end{figure}

An example of an instantiation of the depicted relational schema can be seen in figure \ref{fig:skeleton-SI}. For simplicity, attributes are left with the original placeholder for each entity and relationship instance. As an example, the skeleton contains three instantiations of the USER entity, Bob, Anna, and Andrea, and four instantiations of the POST entity type, Food recipe, Meme, Poem, and News. Bob and Anna react to the Food recipe and Meme, while Andrea reacts to the Poem and News. It is important to note that this skeleton is coherent with the cardinality requirements (i.e., MANY TO MANY) of the relationship defined in the schema.

\begin{figure}[ht]
    \centering
    \begin{tikzpicture}
        \node[draw, rounded corners, rectangle, minimum width=2cm, minimum height=2cm, thick] (E1) at (0, 0) {};
        \node[above] at (E1.north) {Bob};

        \node[draw, rounded corners, rectangle, minimum width=2cm, minimum height=2cm, thick] (E2) at (5, 0) {};
        \node[above] at (E2.north) {Anna};
    
        \node[draw, rounded corners, rectangle, minimum width=2cm, minimum height=2cm, thick] (P1) at (0, -5) {};
        \node[above] at (P1.north) {Food Recipe};

        \node[draw, rounded corners, rectangle, minimum width=2cm, minimum height=2cm, thick] (P2) at (5, -5) {};
        \node[above] at (P2.north) {Meme};

        \coordinate (N1) at (2.5, -1.5);
        \coordinate (ES1) at (3.55, -2.25);
        \coordinate (S1) at (2.5, -3);
        \coordinate (W1) at (1.45, -2.25);
        \node[above] at (N1.north) {Reacts};
    
        \node[draw, ellipse, thick, font=\scriptsize] (A1) at ([yshift=-0.66cm, xshift=-0.09cm] A.center) {Sentiment};
        \node[draw, ellipse, thick, font=\scriptsize] (A2) at ([xshift=-0.45cm, yshift=0.65cm] A.center) {Type};
        \node[draw, ellipse, thick, font=\scriptsize] (A3) at ([xshift=0.24cm, yshift=0cm] A.center) {Activity};

        \node[draw, ellipse, thick, font=\scriptsize] (A1) at ([yshift=-0.66cm, xshift=-0.09cm] E2.center) {Sentiment};
        \node[draw, ellipse, thick, font=\scriptsize] (A2) at ([xshift=-0.45cm, yshift=0.65cm] E2.center) {Type};
        \node[draw, ellipse, thick, font=\scriptsize] (A3) at ([xshift=0.24cm, yshift=0cm] E2.center) {Activity};
    
        \node[draw, ellipse, font=\scriptsize, thick] (loc_1) at ([yshift=-0.75cm] N1.center) {Frequency};
            
        \node[draw, ellipse, thick, font=\tiny]  (B1) at ([xshift=0cm, yshift=-0.45cm] P1.center) {Engagement};
        \node[draw, ellipse, thick, font=\scriptsize] (B2) at ([xshift=0.1cm, yshift=0.5cm] P1.center) {Content};

        \node[draw, ellipse, thick, font=\tiny]  (B1) at ([xshift=0cm, yshift=-0.45cm] P2.center) {Engagement};
        \node[draw, ellipse, thick, font=\scriptsize] (B2) at ([xshift=0.1cm, yshift=0.5cm] P2.center) {Content};

        \draw[thick] (N1) -- (ES1) -- (S1) -- (W1) -- cycle;

        \draw (W1) -- (E1);
        \draw (ES1) -- (E2);
        \draw (W1) -- (P1);
        \draw (ES1) -- (P2);


        \node[draw, rounded corners, rectangle, minimum width=2cm, minimum height=2cm, thick] (E3) at (12.5, 0) {};
        \node[above] at (E3.north) {Andrea};
    
        \node[draw, rounded corners, rectangle, minimum width=2cm, minimum height=2cm, thick] (P3) at (10, -5) {};
        \node[above] at (P3.north west) {Poem};

        \node[draw, rounded corners, rectangle, minimum width=2cm, minimum height=2cm, thick] (P4) at (15, -5) {};
        \node[above] at (P4.north east) {News};

        \coordinate (N2) at (10, -1.75);
        \coordinate (ES2) at (11, -2.5);
        \coordinate (S2) at (10, -3.25);
        \coordinate (W2) at (9, -2.5);
        \node[above] at (N2.north) {Reacts};

        \coordinate (N3) at (15, -1.75);
        \coordinate (ES3) at (16, -2.5);
        \coordinate (S3) at (15, -3.25);
        \coordinate (W3) at (14, -2.5);
        \node[above] at (N3.north) {Reacts};
    
        \node[draw, ellipse, thick, font=\scriptsize] (A1) at ([yshift=-0.66cm, xshift=-0.09cm] E3.center) {Sentiment};
        \node[draw, ellipse, thick, font=\scriptsize] (A2) at ([xshift=-0.45cm, yshift=0.65cm] E3.center) {Type};
        \node[draw, ellipse, thick, font=\scriptsize] (A3) at ([xshift=0.24cm, yshift=0cm] E3.center) {Activity};
    
        \node[draw, ellipse, font=\scriptsize, thick] (loc_2) at ([yshift=-0.75cm] N2.center) {Frequency};

        \node[draw, ellipse, font=\scriptsize, thick] (loc_3) at ([yshift=-0.75cm] N3.center) {Frequency};
            
        \node[draw, ellipse, thick, font=\tiny]  (B1) at ([xshift=0cm, yshift=-0.45cm] P3.center) {Engagement};
        \node[draw, ellipse, thick, font=\scriptsize] (B2) at ([xshift=0.1cm, yshift=0.5cm] P3.center) {Content};

          \node[draw, ellipse, thick, font=\tiny]  (B1) at ([xshift=0cm, yshift=-0.45cm] P4.center) {Engagement};
        \node[draw, ellipse, thick, font=\scriptsize] (B2) at ([xshift=0.1cm, yshift=0.5cm] P4.center) {Content};

        \draw[thick] (N2) -- (ES2) -- (S2) -- (W2) -- cycle;
        \draw[thick] (N3) -- (ES3) -- (S3) -- (W3) -- cycle;

        \draw (ES2) -- (E3);
        \draw (W3) -- (E3);
        \draw (S2) -- (P3);
        \draw (S3) -- (P4);

    \end{tikzpicture}
    \caption{Example of Relational Skeleton}
    \label{fig:skeleton-SI}
\end{figure}
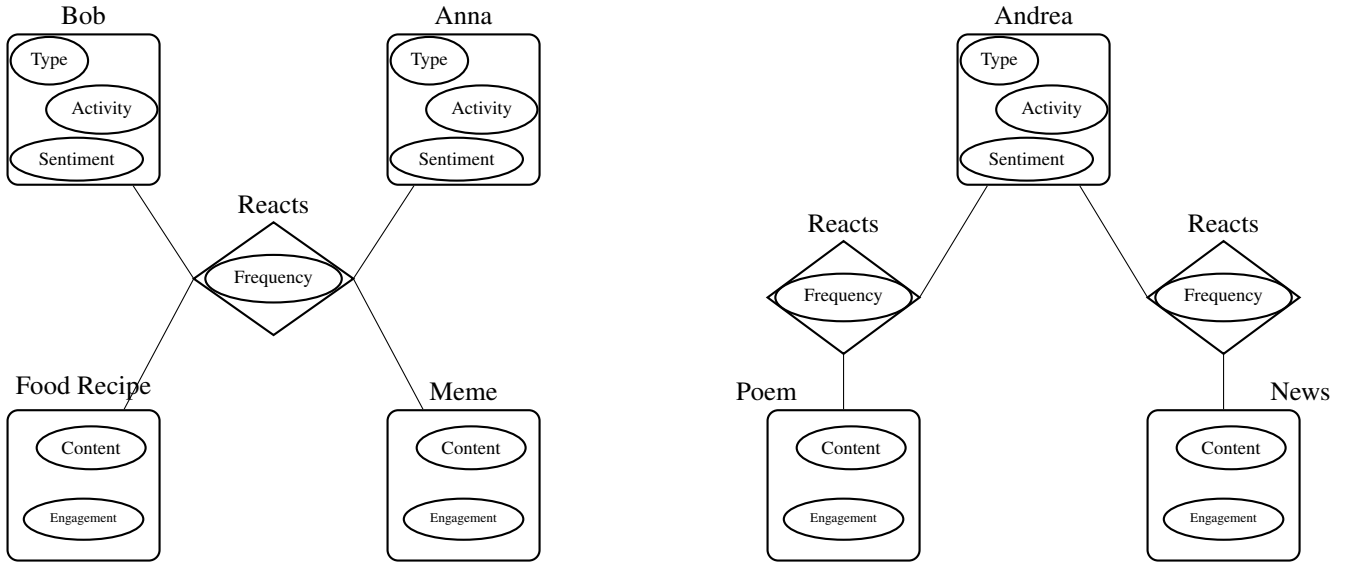

Given the relational skeleton provided and the relational dependencies provided in the relational causal model in figure \ref{fig:model}, it is possible to obtain the corresponding ground graph, shown in figure \ref{fig:groudgraph-SI}. The nodes on the ground graph represent the attributes of every single entity and relationship instance in the skeleton. In contrast, the edges represent the dependencies in the relational causal model applied to the attribute instances of the relational skeleton. For example, the relational dependency $[P, R, U].Sentiment \rightarrow [P].Engagement$ in the model, which indicates that a post's engagement depends on the user's reaction to the product, is represented in the ground graph with the following edges: Bob.Sentiment $\rightarrow$ Food\_Recipe.Engagement, Bob.Sentiment $\rightarrow$ Meme.Engagement, Anna.Sentiment $\rightarrow$ Food\_Recipe.Engagement, Anna.Sentiment $\rightarrow$ Food\_Recipe.Engagement, Andrea.Sentiment $\rightarrow$ Poem.Engagement, Andrea.Sentiment $\rightarrow$ News.Engagement. 

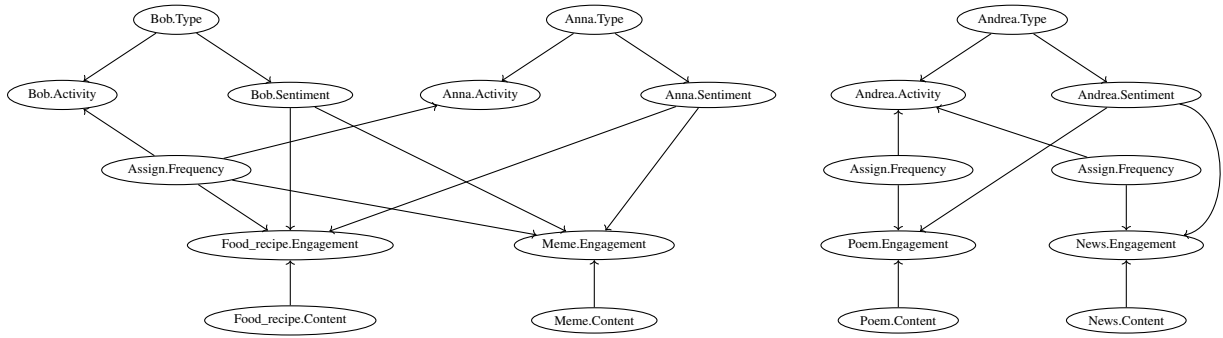
\begin{figure}[ht]
    \centering
    \scalebox{0.5}{
        \begin{tikzpicture}
            \node[draw, ellipse, thick] (1) at (0, 0) {Bob.Type};
            \node[draw, ellipse, thick] (2) at (-3, -2) {Bob.Activity};
            \node[draw, ellipse, thick] (3) at (3, -2) {Bob.Sentiment};
            \node[draw, ellipse, thick] (4) at (0, -4) {Assign.Frequency};
            \node[draw, ellipse, thick] (5) at (3, -6) {Food\_recipe.Engagement};
            \node[draw, ellipse, thick] (6) at (3, -8) {Food\_recipe.Content};
            \node[draw, ellipse, thick] (7) at (11, 0) {Anna.Type};
            \node[draw, ellipse, thick] (8) at (8, -2) {Anna.Activity};
            \node[draw, ellipse, thick] (9) at (14, -2) {Anna.Sentiment};
            \node[draw, ellipse, thick] (10) at (11, -6) {Meme.Engagement};
            \node[draw, ellipse, thick] (11) at (11, -8) {Meme.Content};

            \node[draw, ellipse, thick] (12) at (22, 0) {Andrea.Type};
            \node[draw, ellipse, thick] (13) at (19, -2) {Andrea.Activity};
            \node[draw, ellipse, thick] (14) at (25, -2) {Andrea.Sentiment};
            \node[draw, ellipse, thick] (15) at (19, -4) {Assign.Frequency};
            \node[draw, ellipse, thick] (16) at (25, -4) {Assign.Frequency};
            \node[draw, ellipse, thick] (17) at (19, -6) {Poem.Engagement};
            \node[draw, ellipse, thick] (18) at (19, -8) {Poem.Content};
            \node[draw, ellipse, thick] (19) at (25, -6) {News.Engagement};
            \node[draw, ellipse, thick] (20) at (25, -8) {News.Content};

            \path (1) edge[thick, ->] (2);
            \path (1) edge[thick, ->] (3);
            \path (3) edge[thick, ->] (5);
            \path (3) edge[thick, ->] (10);
            \path (4) edge[thick, ->] (2);
            \path (4) edge[thick, ->] (5);
            \path (4) edge[thick, ->] (8);
            \path (4) edge[thick, ->] (10);
            \path (6) edge[thick, ->] (5);
            \path (7) edge[thick, ->] (8);
            \path (7) edge[thick, ->] (9);
            \path (9) edge[thick, ->] (5);
            \path (9) edge[thick, ->] (10);
            \path (11) edge[thick, ->] (10);

            \path (12) edge[thick, ->] (13);
            \path (12) edge[thick, ->] (14);
            \path (14) edge[thick, ->] (17);
            \path (14) edge[thick, bend left=80, ->] (19);
            \path (15) edge[thick, ->] (13);
            \path (15) edge[thick, ->] (17);
            \path (16) edge[thick, ->] (13);
            \path (16) edge[thick, ->] (19);
            \path (18) edge[thick, ->] (17);
            \path (20) edge[thick, ->] (19);

        \end{tikzpicture}
    }
    \caption{Example of Ground Graph}
    \label{fig:groudgraph-SI}
\end{figure}

Beyond the specific instantiation of the ground graph, to perform relational causal discovery, it is necessary to define an abstract ground graph that generalizes the structure of dependencies without referring to particular entities or relationship instances. The abstract ground graph represents the relational dependencies in the relational causal model at a higher level, capturing attribute interactions without being tied to a specific skeleton. In this representation, nodes correspond to attribute types rather than instances, while edges represent the abstract relational dependencies in the model \citep{maier2014reasoning}. 

To construct the abstract ground graph from a given relational causal model, it is necessary to project the dependencies onto the relevant perspective. The \verb|extend| method devised by \citet{maier2014reasoning} achieves this by mapping underlying relational dependencies into the set of edges in the abstract ground graphs. Below we provide the formula of the \verb|extend| method:
\begin{align*}
    \text{extend}(P_{\text{orig}}, P_{\text{ext}}) &= \left\{ 
P = P_{\text{orig}}^{1, n_o - i + 1} + P_{\text{ext}}^{i+1, n_e} 
\,\middle|\, 
i \in \text{pivots}(\text{reverse}(P_{\text{orig}}), P_{\text{ext}}) 
\wedge \text{validPath}(P) 
\right\} \\
\\
\text{pivots}(P_1, P_2) &= \left\{ i \,\middle|\, P_1^{1,i} = P_2^{1,i} \right\}
\end{align*}
Where $\text{validPath}(P)$ checks that the relational path is valid with the respect to the schema and its relationships' cardinalities. \\
Each abstract ground graph edge $
[B, \ldots, I_k].Y \rightarrow [B, \ldots, I_j].X$ is then constructed from the underlying dependency $[I_j, \ldots, I_k].Y \rightarrow [I_j].X$ with the following logic:
\begin{align*}
    \left\{
[B, \ldots, I_k].Y \rightarrow [B, \ldots, I_j].X \,\middle|\,
[I_j, \ldots, I_k].Y \rightarrow [I_j].X \in \mathcal{D} \,\wedge\,
[B, \ldots, I_k] \in \text{extend}([B, \ldots, I_j], [I_j, \ldots, I_k])
\right\}
\end{align*}
For example, the relational dependency $[P, R, U].Sentiment \rightarrow [P].Engagement$, which in the ground graph manifests as instance-specific edges (e.g., Bob.Sentiment $ \rightarrow $ Food\_Recipe.Engagement), is represented in the abstract ground graph for the perspective USER with the directed edges $[U].Sentiment \rightarrow [U, R, P].Engagement$ and $[U, R, P, R, U].Sentiment \rightarrow [U, R, P].Engagement$. 

Similarly, other relational dependencies in the model are reflected as edges between attribute types in the abstract ground graph, providing a compact and generalized view of how information propagates through the relational structure. Analyzing the abstract ground graph makes it possible to reason about potential influences and dependencies at the schema level without requiring explicit enumeration of individual instances.

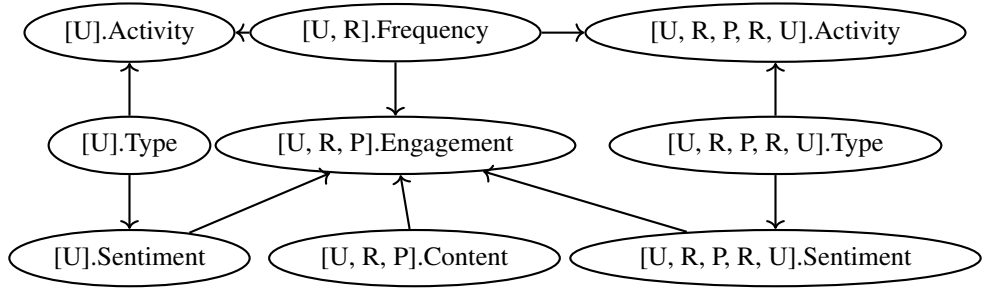
\begin{figure}[h!]
    \centering
        \begin{tikzpicture}
            \node[draw, ellipse, thick] (1) at (-1, -1.5) {[U].Sentiment};
            \node[draw, ellipse, thick] (2) at (-1, 0) {[U].Type};
            \node[draw, ellipse, thick] (3) at (-1, 1.5) {[U].Activity};
            \node[draw, ellipse, thick] (4) at (2.5, 0) {[U, R, P].Engagement};
            \node[draw, ellipse, thick] (5) at (2.75, -1.5) {[U, R, P].Content};
            \node[draw, ellipse, thick] (6) at (7.5, -1.5) {[U, R, P, R, U].Sentiment};
            \node[draw, ellipse, thick] (7) at (7.5, 0) {[U, R, P, R, U].Type};
            \node[draw, ellipse, thick] (8) at (7.5, 1.5) {[U, R, P, R, U].Activity};
            \node[draw, ellipse, thick] (9) at (2.5, 1.5) {[U, R].Frequency};

            \path (2) edge[thick, ->] (1);
            \path (2) edge[thick, ->] (3);
            \path (9) edge[thick, ->] (3);
            \path (9) edge[thick, ->] (4);
            \path (1) edge[thick, ->] (4);
            \path (5) edge[thick, ->] (4);
            
            \path (7) edge[thick, ->] (6);
            \path (7) edge[thick, ->] (8);
            \path (9) edge[thick, ->] (8);
            \path (6) edge[thick, ->] (4);
        \end{tikzpicture}
    \caption{Example of Abstract Ground Graph for perspective USER}
    \label{fig:agg}
\end{figure}
\subsection{MAGs and PAGs}\label{fci}
In this subsection, we provide an example of how more than one MAG can be a member of the same PAG and single equivalency class. Given a collection of observable variables, let \textbf{Cond} in figure \ref{fig:pag1} represent the set of conditional dependencies. It is evident that it is entailed by several DAGs. Figure \ref{fig:pag2} displays the PAG that was generated for \textbf{Cond}. Since they are not mentioned in the conditional set, A and D's edge marks are ◦, which could lead to different marks for various DAGs in \textit{O-Equiv}(\textbf{Cond}).
\begin{figure}[h!]
    \centering
    \begin{subfigure}{0.45\columnwidth}
        \centering
        \scalebox{0.7}{
            \begin{tikzpicture}
                \node (A1) at (0,0) {A};
                \node (B1) at (0.8,0) {B};
                \node (C1) at (1.6,0) {C};
                \node (D1) at (2.4,0) {D};
                \node[draw] (L1_1) at (1.2, 0.8) {L1};
    

                \node (A2) at (3,0) {A};
                \node (B2) at (3.8,0) {B};
                \node (C2) at (4.6,0) {C};
                \node (D2) at (5.4,0) {D};
                \node[draw] (L1_2) at (4.2, 0.8) {L1};
                \node[draw] (L2_2) at (4.2, -0.8) {L2};
            
                \draw[thick, -{Stealth[round]}] (A1) -- (B1);
                \draw[thick, -{Stealth[round]}] (L1_1) -- (B1);
                \draw[thick, -{Stealth[round]}] (L1_1) -- (C1);
                \draw[thick, -{Stealth[round]}] (D1) -- (C1);
    
                \draw[thick, -{Stealth[round]}] (A2) -- (B2);
                \draw[thick, -{Stealth[round]}] (L1_2) -- (B2);
                \draw[thick, -{Stealth[round]}] (L1_2) -- (C2);
                \draw[thick, -{Stealth[round]}] (L2_2) -- (B2);
                \draw[thick, -{Stealth[round]}] (L2_2) -- (C2);
                \draw[thick, -{Stealth[round]}] (D2) -- (C2);
    
                \node[align=left] at (2.7,2) {\{ \{D\} $\perp$ \{A, B\},\\\hspace{0.26cm}\{A\} $\perp$ \{C, D\} \}};
            \end{tikzpicture}
        }
        \caption{DAGs in same O-Equiv(Cond) class}
        \label{fig:pag1}
    \end{subfigure}%
    \hfill
    \begin{subfigure}{0.45\columnwidth}
        \centering
        \begin{tikzpicture}
            \node (A) at (0,0) {A};
            \node (B) at (0.5, -1.5) {B};
            \node (C) at (2, -1.5) {C};
            \node (D) at (2.5,0) {D};
        
            \draw[thick, {Circle[open]}-{Stealth[round]}] (A) -- (B);
            \draw[thick, {Stealth[round]}-{Stealth[round]}] (B) -- (C);
            \draw[thick, {Circle[open]}-{Stealth[round]}] (D) -- (C);
        \end{tikzpicture}
        
        \caption{Resulting PAG for O-Equiv(Cond) class}
        \label{fig:pag2}
    \end{subfigure}
    \caption{DAGs in the same observational equivalence class under \textbf{Cond} (a), and the resulting PAG (b) that captures shared structure and uncertainty in edge directions.}
    \label{fig:both_images}
\end{figure}
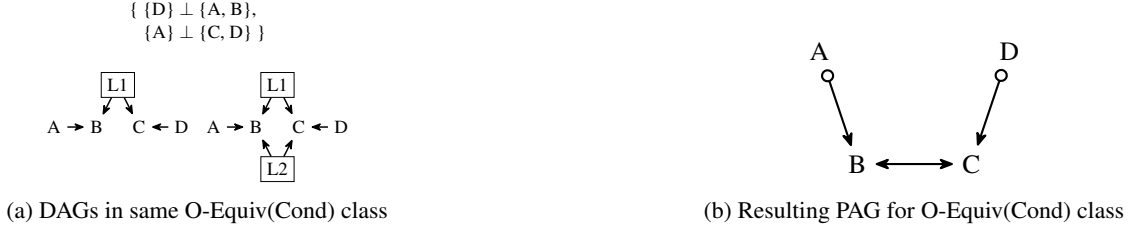

\section{RelFCI Rules}\label{rules}
This section outlines every rule we apply to the new Partial Ancestral Abstract Ground Graph representation to obtain a maximally informative graph and, thus, an underlying model. We introduce the rules in the framework of PAAGGs, where any ◦ marks represent unoriented edges and $\ast$ denotes any edge mark.
\subsection{RCD Rules}
RCD \citep{maier2013sound} performs relational causal discovery using a similar strategy to the Poem algorithm, extended with the RBO purely common cause rule.
The edges of the abstract ground graph are oriented using the following set of rules:
\begin{enumerate}
    \item Collider Detection (CD): For each triple $\langle\alpha,\beta,\gamma\rangle$, if $\beta$ is not in the set that separates $\alpha$ and $\gamma$, orient it as $\alpha\rightarrowast\beta\leftarrowast\gamma$;
    \item Relational Bivariate Orientation (RBO): Let $\mathcal{M}$ be a relational causal model  and $G$ a partially directed PAAGG for $\mathcal{M}$ for perspective $I_X$, and let there be an unshielded triple in $G$ $\alpha$◦\textemdash◦$\beta$◦\textemdash◦$\gamma$ with $\alpha=[I_X].X, \beta=[I_X,...,I_Y].Y, \gamma=[I_X,...,I_Y,...,I_X].X$. If $\textit{card}([I_Y,...,I_X])=\verb|MANY|$ and $\alpha\independent\gamma\rvert\mathbf{Z}$, then if $\beta\in\textbf{Z}$, orient the triple as $\alpha\leftarrowcircle\beta\rightarrowcircle\gamma$;
    \item Known Non-Colliders (KNC): If $\alpha\rightarrowast\beta$◦\textemdash$\ast\gamma$, with $\alpha,\gamma$ not adjacent, orient the triple as $\alpha\rightarrowast\beta\rightarrow\gamma$
    \item Cycle Avoidance (CA): If either $\alpha\rightarrow\beta\rightarrowast\gamma$ or $\alpha\rightarrowast\beta\rightarrow\gamma$, with $\alpha\ast$\textemdash◦$\gamma$, orient the latter as $\alpha\rightarrowast\gamma$;
    \item Meek Rule 3 (MR3): If both $\alpha\rightarrowast\beta\leftarrowast\gamma$ and $\alpha\ast$\textemdash◦$\theta$◦\textemdash$\ast\gamma$, with $\alpha,\gamma$ not adjacent and $\theta\ast$\textemdash◦$\beta$, then orient the latter as $\theta\rightarrowast\beta$.
\end{enumerate}
\subsection{FCI Rules}

FCI \citep{ZHANG20081873} constructs a causal graph starting from a fully connected undirected graph with ◦ marks and removes edges between conditionally dependent variables. In the second phase, it orients edges by identifying colliders and "Y" structures. The remaining edges are then oriented according to a set of additional rules: 
\begin{enumerate}
    \setcounter{enumi}{3}
    \item If $u=\langle\theta,...,\alpha,\beta,\gamma\rangle$ is a discriminating path and $\beta\ast$\textemdash$\gamma$, if $\beta\in\textit{SepSet}(\theta,\gamma)$ orient $\beta\rightarrow\gamma$, otherwise orient $\alpha\leftrightarrow\beta\leftrightarrow\gamma$;
    \item For every (remaining) $\alpha$◦\textemdash◦$\beta$, if there is an uncovered path $p = \langle\alpha, \gamma,...,\theta,\beta\rangle$ s.t. all edges are ◦\textemdash◦ and $\alpha, \theta$ are
    not adjacent and $\beta,\gamma$ are not adjacent, then orient all edges in the path as \textemdash;
    \item If $\alpha$\textemdash$\beta$◦\textemdash$\ast\gamma$, with $\alpha, \gamma$ either adjacent or not, orient $\beta$\textemdash$\ast\gamma$;
    \item If $\alpha$\textemdash◦$\beta$◦\textemdash$\ast\gamma$, and $\alpha, \gamma$ are not adjacent, orient $\beta$\textemdash$\ast\gamma$;
    \item If $\alpha$\textemdash◦$\beta\rightarrow\gamma$ or $\alpha$\textemdash◦$\beta\rightarrow\gamma$, and $\alpha$◦$\rightarrow\gamma$, orient  $\alpha\rightarrow\gamma$;
    \item If $\alpha\rightarrowcircle\gamma$ and $p = \langle\alpha,\beta,  \theta,...,\gamma\rangle$ is an uncovered path s.t. $\beta$ and $\gamma$ are not adjacent, orient $\alpha\rightarrow\gamma$;
    \item If $\alpha\rightarrowcircle\gamma$, $\beta\rightarrow\gamma\leftarrow\theta$, and $p_1,p_2$ are uncovered p.d. paths from $\alpha$ to $\beta$ and from $\alpha$ to $\theta$, let $\mu$ and $\omega$ be the adjacent nodes of $\alpha$ on $p_1,p_2$. If $\mu$ and $\omega$ are distinct, orient $\alpha\rightarrow\gamma$.
\end{enumerate}

\section{Algorithms}\label{algo}
The following section provides more detailed pseudocode for each step in the main algorithm. The described algorithm and steps are adapted from the implementation provided in \citet{colombo2012learning}. For easy reference, the main RelFCI pseudocode is provided again below in Algorithm \ref{alg:main}.

\begin{algorithm}[htb]
\caption{RelFCI algorithm}
\label{alg:main}
\textbf{Input}: schema, oracle,\\
\textbf{Parameter}: threshold\\
\textbf{Output}: Dependencies
\begin{algorithmic}[1] 
\STATE \textit{// Step 1: Graphs initialization}
\STATE $PDs \gets$ get potential Dependencies from the base schema (with no dependencies) and two times the threshold (2*h)
\STATE $PAAGGs \gets$ construct PAAGGs from potential dependencies set $PDs$
\STATE $S \gets \{\}$ \\
\STATE \textit{// Step 1: Independent Variables identification, storing separating sets and unshielded triples}
\STATE $PAAGGs, S, U \gets \text{obtainInitialSkeleton}(PAAGGs, S)$ 
\STATE \textit{// Step 2: V-structures orientation using CD, starting from unshielded triples in $U$}
\STATE $PAAGGS, S \gets \text{orientVStructures}(PAAGGs, S, U)$  
\STATE \textit{// Step 3: edges orientation using rules from RCD and additional ones from FCI}
\STATE $PAAGGs, S \gets \text{performEdgeOrientation}(PAAGGs, S)$
\STATE $Deps \gets$ retrieve underlying dependencies from oriented PAAGGs edges
\STATE \textbf{return} Deps
\end{algorithmic}
\end{algorithm}

\begin{algorithm}[htb]
\caption{obtainInitialSkeleton}
\label{alg:step_1}
\textbf{Input}: Schema, Oracle,\\
\textbf{Parameter}: threshold, depth\\
\textbf{Output}: Non-oriented AGGs
\begin{algorithmic}[1] 

\FOR{$agg$ \textbf{in} AGGs}
    \STATE Let $l=0$
    \STATE Let $max\_depth=agg.number\_of\_nodes - 2$
    \WHILE{$l \leq max\_depth$}
        \FORALL{pair of vertices ($X_i$, $X_j$) in $agg$}
            \STATE Let $C = agg.nodes - \{X_i, X_j\}$
            \FORALL{$Y \subseteq C$}
                \IF {CITest($X_i$, $X_j$, $Y$)}
                    \STATE Remove dependencies between ($X_i$, $X_j$)
                    \STATE Store $Y$ as $sepSet$ for ($X_i$, $X_j$)
                \ENDIF
            \ENDFOR
        \ENDFOR
        \STATE Let $l = l + 1$
    \ENDWHILE

    \STATE
    \FORALL{triple of vertices ($X_k$, $X_j$, $X_m$) in $agg$}
        \IF{$k < m$}
            \IF{$agg.has\_edge(X_k, X_j)$ \textbf{and} $agg.has\_edge(X_j, X_m)$ \textbf{and not} $agg.has\_edge(X_k, X_m)$}
                \STATE Append ($X_k$, $X_j$, $X_m$) to $unshieldedTriples[agg]$ 
            \ENDIF
        \ENDIF
    \ENDFOR
\ENDFOR
\end{algorithmic}
\end{algorithm}

\begin{algorithm}[htb]
\caption{orientVStructures}
\label{alg:step_2}
\textbf{Input}: Schema, Oracle,\\
\textbf{Parameter}: threshold, depth\\
\textbf{Output}:  Partially oriented AGGs
\begin{algorithmic}[1] 
\FOR{$agg$ \textbf{in} AGGs}
    \WHILE{$unshieldedTriples[agg]$}
        \STATE Let $(X_i, X_j, X_k) = unshieldedTriples[agg].pop()$
        \STATE Let $Z = sepSet(X_i, X_k) - \{ X_j\}$
        \IF {\textbf{not} CITest($X_i$, $X_j$, $Z$) \textbf{and} \textbf{not} CITest($X_j$, $X_k$, $Z$)}
            \STATE Append ($X_i$, $X_j$, $X_k$) to $dependentTriples[agg]$ 
        \ELSE
            \FOR{$X_r$ \textbf{in} [$X_i$, $X_k$]}
                \IF {CITest($X_r$, $X_j$, $Z$)}
                    \STATE Let $Y = findMinimalSepset(X_r, X_j, Z)$
                    \STATE Store $Y$ as $sepSets$ for ($X_r$, $X_j$)
                    \FORALL{$X_x$ \textbf{in} $agg.nodes$}
                        \IF{isTriangle($X_{min(r,j)}, \cdot , X_{max(r,j)}$)}
                            \STATE Add to $unshieldedTriples[agg]$ the triple
                        \ENDIF
                    \ENDFOR
                    \FORALL{triple \textbf{in} $unshieldedTriples[agg]$}
                        \STATE Delete the triple if matches one of the following patterns: $(X_r, X_j, \cdot )$, $(X_j, X_r, \cdot )$, $( \cdot, X_j, X_r)$ and $(\cdot, X_r, X_j)$
                    \ENDFOR
                    \STATE Remove dependencies between ($X_r$, $X_j$)
                \ENDIF
            \ENDFOR
        \ENDIF 
    \ENDWHILE

    \FORALL{triple \textbf{in} $dependentTriples[agg]$}
        \STATE Let $X_i, X_j, X_k = triple$
        \IF{$X_j$ \textbf{not in} $sepSets(X_i, X_k)$ \textbf{and} $agg.has\_edge(X_i, X_j)$ \textbf{and} $agg.has\_edge(X_j, X_k)$}
            \STATE Orient the triple as a collider
        \ENDIF
    \ENDFOR
\ENDFOR
\end{algorithmic}
\end{algorithm}

\begin{algorithm}[htb]
\caption{performEdgeOrientation}
\label{alg:step_3}
\textbf{Input}: Schema, Oracle,\\
\textbf{Parameter}: threshold, depth\\
\textbf{Output}:  Maximum oriented AGGs
\begin{algorithmic}[1] 
\FOR{$agg$ \textbf{in} AGGs}
    \WHILE{AGG is updated}
        \STATE Orient as many edges as possible by applying RBO rule
        \STATE Orient as many edges as possible by applying FCI\_1 - FCI\_3 rules
        \FORALL{possible triples}
            \STATE Let $X_l, X_j, X_k = triple$
            \IF{$isTriangle(X_l, X_j, X_k)$ \textbf{and} $X_j  \circleast X_k$ \textbf{and} $X_l \leftarrowast X_j$ and $X_l \rightarrow X_k$}
                \STATE Find Minimal Discriminating Path for the triple
                \IF{minimalDiscriminatingPath}
                    \FORALL{adjacent couples}
                        \STATE Let $X_r, X_q = couple$
                        \STATE Let $otherSepSet = sepSets(X_i, X_k) - {X_r, X_q}$
                        \STATE Let $l = -1$
                        \WHILE{$|otherSepSet| \geq l$}
                            \STATE Let $l = l + 1$
                            \FORALL{$Y \subseteq otherSepSet$ \textbf{and} $|Y| = l$}
                                \IF{CITest($X_r$, $X_q$, $Y$)}
                                    \STATE Store $Y$ as $sepSet$ for ($X_r$, $X_q$)
                                    \FORALL{$X_x$ \textbf{in} $agg.nodes$}
                                        \IF{isTriangle($X_{min(r,j)}, \cdot , X_{max(r,j)}$)}
                                            \STATE Add to $unshieldedTriples[agg]$ the triple
                                        \ENDIF
                                    \ENDFOR
                                    \STATE Remove dependencies between ($X_r$, $X_q$)
                                    \STATE Execute Algorithm 2
                                \ENDIF
                            \ENDFOR
                        \ENDWHILE
                    \ENDFOR
                    \IF{Still adjacent \textbf{and} $X_j$ \textbf{in} $sepSets(X_i, X_k)$}
                        \STATE Orienting $X_j \rightarrow X_k$
                    \ELSIF{Still adjacent}
                        \STATE Orienting $X_l \leftrightarrow X_j \leftrightarrow X_k$
                    \ENDIF
                \ENDIF
            \ENDIF
        \ENDFOR
        \STATE Orient as many edges as possible by applying FCI\_5 - FCI\_10 rules
    \ENDWHILE
\ENDFOR

\end{algorithmic}
\end{algorithm}

\clearpage

\section{Possible Dependencies}\label{deps}
The presence of ◦ marks in the edge of $PAAGGs$, and thus in the underlying $PARM$, implies that the \textit{O-Equiv}($\mathcal{D}_{\textbf{O}}$) class contains different relational causal models. The algorithm's output is not the exact relational causal model that generates the data. RelFCI returns an equivalence class containing the model responsible for the data causal relationships. 
RelFCI computes conditional independence tests among the variables, thus possibly producing the same result with different underlying topologies e.g., with the independence fact $A \independent C \mid B$, the nodes A, B, and C can be correctly oriented as follows: $A \rightarrow B \rightarrow C$,  $A \leftarrow B \rightarrow C$,  $A \leftarrow B \leftarrow C$, $ A \rightarrow B \leftarrow C$ \citep{spirtes2000causation}.
RelFCI works by learning the edges' orientation of each $PAAGG$, which are defined by underlying relational dependencies. 

When the algorithm concludes and collects all the information learned to produce the $PARM$, the remaining ◦ marks lose significance in terms of relational dependencies. 
The definition of relational dependency in canonical form implies a natural orientation, i.e., $[I_X...I_Y].Y \rightarrow [I_X].X$. Orienting dependencies the other way around is an infraction of the definition, i.e., $[I_X].X \nrightarrow [I_X...I_Y].Y$. 
For this reason, given this formalization of the problem, we differentiate the information the algorithm learns by clearly stating which relational dependencies are required to define the $PARM$ and which are instead allowed. 
We define the required relational dependencies with a $\rightarrow$, i.e., $[I_X...I_Y].Y \rightarrow [I_X].X$ and the ones that are allowed but not necessary with a $\leadsto$, i.e., $[I_X...I_Y].Y \leadsto [I_X].X$. We will refer to the latter as \textit{Possible Dependencies}.

\section{PAAGG edge orientation}\label{edges}
We apply the four PC rules and the new RBO rule, described in RCD, and further apply the rules of FCI, as defined by Zhang (2008), adapted for the PAAGG representation. A latent relational causal model consists of a set of AGGs, one for each perspective, derived from the same set of relational dependencies $\mathcal{D}$. Similarly, both MAAGGs and PAAGGs are derived from the same collection of observed relational dependencies $\mathcal{D}_\textbf{O}$. In classical AGGs, activating a rule in a certain abstract ground graph involves propagating the orientation of the underlying dependency across all AGGs \citep{maier2013sound}. 

Consider a PARM $\mathbfcal{M}$ defined over the set of dependencies $\mathcal{D}_\textbf{O}$ and its corresponding PAAGG $G$ for the perspective $\mathcal{B}$. Let $\alpha=[\mathcal{B},..., I_X].X$ and $\gamma=[\mathcal{B},..., I_Y].Y$ be two nodes in $G$, $\alpha-\gamma$ be a bidirected edge in $G$, and $d_1=[I_X, ..., I_Y].Y\rightarrow[I_X].X\in\mathcal{D}_\textbf{O}$ be the underlying dependency that yields the left direction of the edge.
The FCI rules can orient a PAG edge with three edge marks: ◦, \textemdash, and $\rightarrow$. We apply these orientations to the PAAGG using the following logic:
\begin{itemize}
    \item The orientation $\alpha$◦\textemdash$\gamma$ implies that the underlying dependency $d_1$ belongs to the set of possible dependencies;
    \item The orientation $\alpha - \gamma$ implies that the underlying dependency $d_1$ is not coherent with the edge orientation and, as such, is not existent in the underlying PARM;
    \item The orientation $\alpha \leftarrow \gamma$ indicates that the underlying dependency $d_1$ is consistent with the edge orientation and belongs to the category of exact dependencies.
\end{itemize}
With this logic, the same propagation property applies to new representations that share the same underlying dependencies because exact and potential dependencies are propagated equally.

\section{Example Execution of RelFCI}
To illustrate the functioning of the RelFCI algorithm, we provide a step-by-step execution over an example relational causal model. This walk-through demonstrates the graphical transformations applied to the Partial Ancestral Abstract Ground Graph (PAAGG) across the different phases of the algorithm. Each figure referenced corresponds to a visual depiction of the model after the respective step of the algorithm. 

\noindent\textbf{Note:} For this example, we focus on a single perspective (in this case, $AB1$). Similar graphs and reasoning are applied to all other perspectives. Rule propagation ensures that orientations in one PAAGG are reflected across others in line with shared underlying dependencies.

\subsection*{Initial Model and Underlying Graph}

\begin{figure}[h!]
    \centering
    \includegraphics[width=0.8\textwidth]{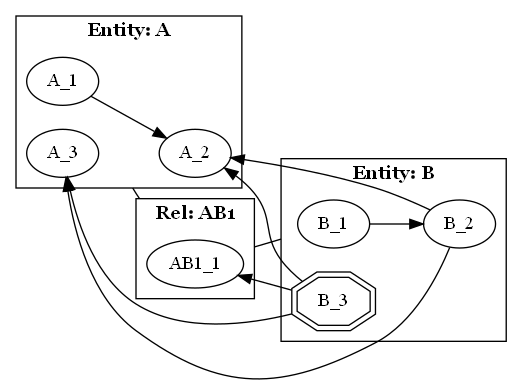}
    \caption{Relational causal model with entities, relationships, and dependencies, including latent variables.}
\end{figure}

We begin with a relational causal model that includes observed and latent variables. The figure depicts:

\begin{itemize}
    \item Entities $A$ and $B$ with a relationship $AB1$;
    \item Attributes $A_1, A_2, A_3, B_1, B_2$ (observed), and $B_3$ (latent, represented with a double edges octagon);
    \item Dependencies between relational variables, considering a hop threshold $h=2$:
        \begin{itemize}
            \item Observed dependencies $\in \mathcal{D}_{\boldsymbol{O}}$: 
                    $[A].A_1 \rightarrow [A].A_2$, 
                     $[A, AB1, B].B_2 \rightarrow [A].A_2$,
                     $[A, AB1, B].B_2 \rightarrow [A].A_3$,
                     $[B].B_1 \rightarrow [B].B_2$].
             \item Unobserved dependencies $\in \mathcal{D}_{\boldsymbol{L}}$:
                    $[A, AB1, B].B_3 \rightarrow [A].A_2$,
                     $[A, AB1, B].B_3 \rightarrow [A].A_3$,
                     $[AB1, B].B_3 \rightarrow [AB1].AB1_1$.
        \end{itemize}
\end{itemize}

\subsection*{Phase 0 – PAAGG Construction}

In this phase, the algorithm constructs the PAAGGs with all possible dependencies:

\begin{itemize}
    \item A node is created for each relational variable with a path length up to the hop threshold $h'=2h=4$.
    \item Edges are added according to the $\texttt{extend}$ method, resulting in a fully connected undirected graph with $\circ{-}\circ$ marks.
    \item Intersection variables are included if needed to maintain the closure under intersections. In this example, these variables are excluded from the plots for better readability.
\end{itemize}

The graph in \ref{fig:phase-0} represents the PAAGG with all potential dependencies for the perspective $AB1$.

\begin{figure}[h!]
    \centering
    \includegraphics[width=\textwidth]{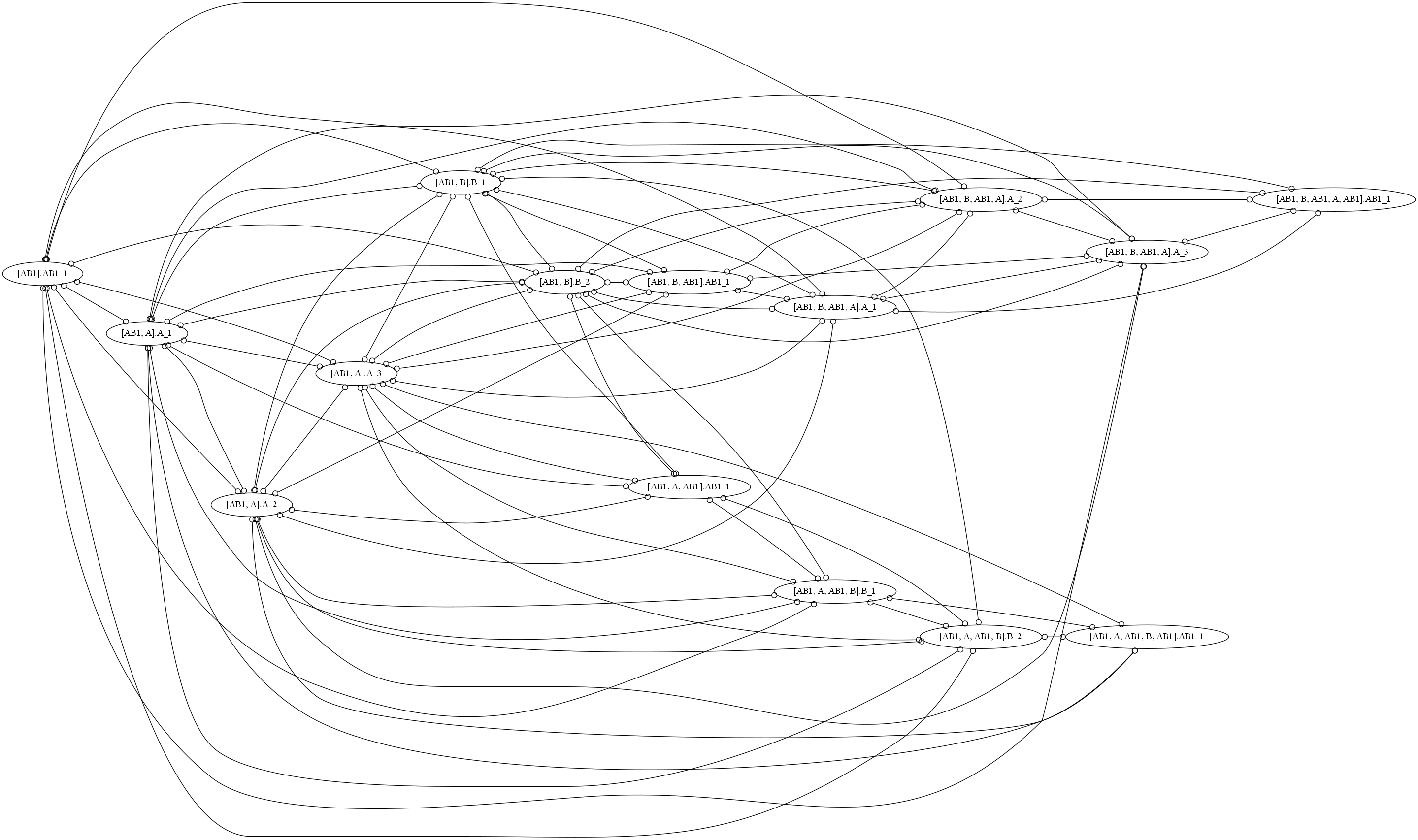}
    \caption{Fully connected PAAGG for perspective $AB1$.}
    \label{fig:phase-0}
\end{figure}

\subsection*{Phase 1 – Initial Skeleton Identification via Conditional Independence Testing}

The algorithm now performs conditional independence tests between every pair of variables, using increasingly bigger separating sets. If the two variables are found to be independent conditioned on the variables in the separating set, the edge is removed, and the set is stored.

\begin{figure}[h!]
    \centering
    \includegraphics[width=\textwidth]{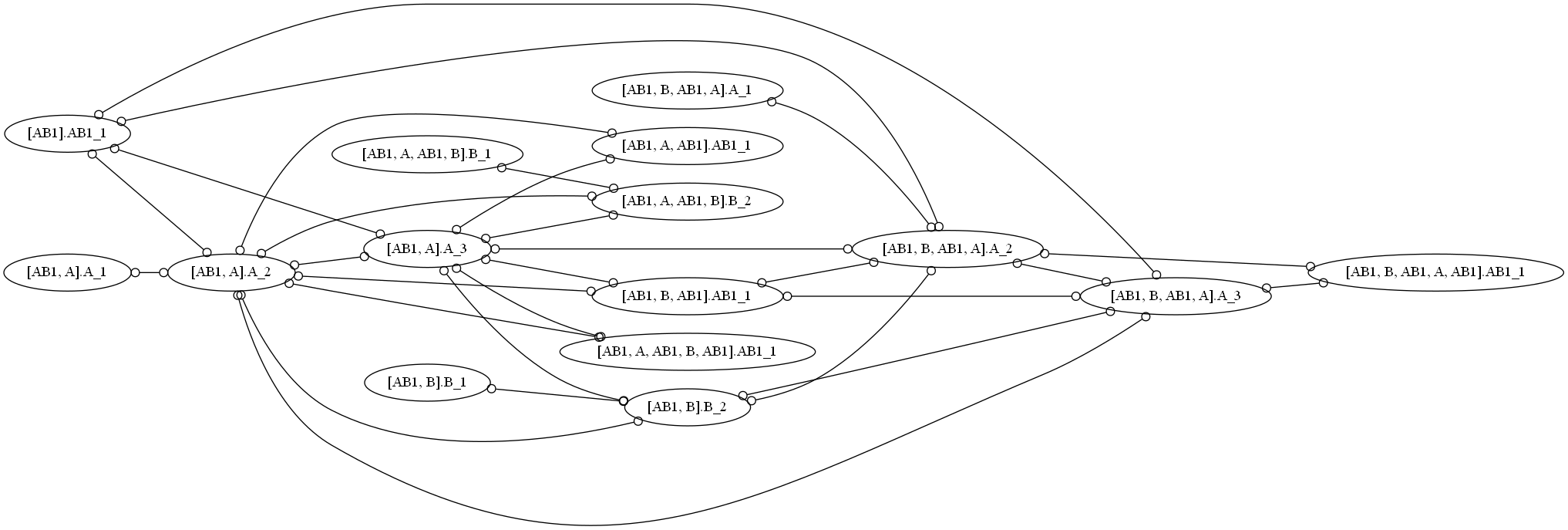}
    \caption{PAAGG after conditional independence testing.}
\end{figure}

Unshielded triples are also identified at this stage as candidate collider patterns. In this example, the following triples are found: 	 
\begin{itemize}
        \item $[AB1].AB1_1, [AB1, A].A_2, [AB1, A].A_1$;
	 \item $[AB1].AB1_1, [AB1, A].A_2, [AB1, B].B_2$;
	 \item $[AB1].AB1_1, [AB1, A].A_2, [AB1, A, AB1, B].B_2$;
	 \item $[AB1].AB1_1, [AB1, A].A_3, [AB1, B].B_2$;
	 \item $[AB1].AB1_1, [AB1, A].A_3, [AB1, A, AB1, B].B_2$;
	 \item $[AB1].AB1_1, [AB1, B, AB1, A].A_2, [AB1, B].B_2$;
	 \item $[AB1].AB1_1, [AB1, B, AB1, A].A_2, [AB1, B, AB1, A].A_1$;
	 \item $[AB1].AB1_1, [AB1, B, AB1, A].A_3, [AB1, B].B_2$.
    \end{itemize}

\subsection*{Phase 2 – Collider Detection and V-Structure Orientation}

This phase introduces the first directed edge orientations in the graph. The algorithm starts by checking whether the unshielded triples are found to be dependent (i.e., for triple $X,Y,Z$, $X,Z$ and $Y,Z$ are not independent given the separating set of $X$ and $Z$) or not. For this example, all 7 unshielded triples are identified as dependent. Then, the CD rule is applied to identify and orient colliders among these triples.
The PAAGG after CD is applied is shown on figure \ref{fig:phase-2}.
\begin{figure}[h!]
    \centering
    \includegraphics[width=\textwidth]{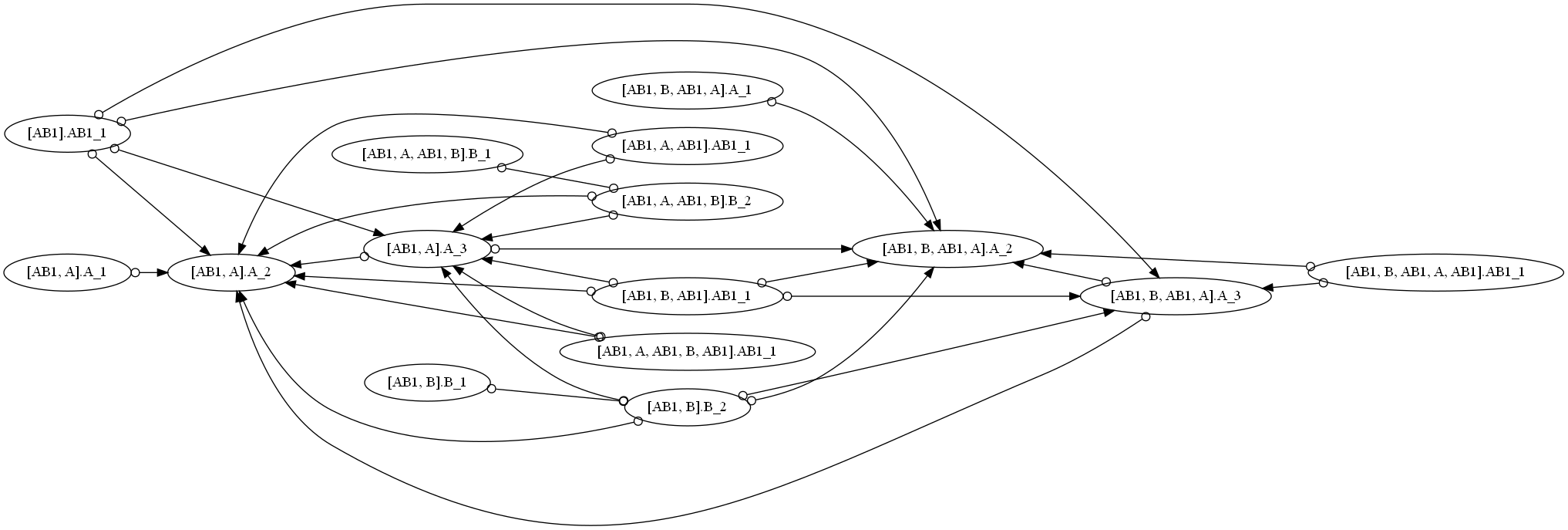}
    \caption{PAAGG after collider orientation via CD.}
    \label{fig:phase-2}
\end{figure}

\subsection*{Phase 3 – Further Orientation via RCD and FCI Rules}

In this step, remaining ambiguous edge marks are refined using the additional RCD (RBO, CA, MR3, and KNC) and FCI rules, repeating this process until no rule can be applied anymore. For this example:
\begin{itemize}
    \item Rule KNC is activated once to orient the triple $[AB1, A, AB1].AB1_1 \rightarrowast [AB1, A].A_3 \rightarrow [AB1, B, AB1, A].A_2$ and all other triples sharing the same underlying dependencies;
    \item FCI rule R4 is activated once to orient the triangle $[AB1, A].A_3 \leftrightarrow [AB1].AB1_1 \leftrightarrow [AB1, B, AB1, A].A_2$ and all other triples sharing the same underlying dependencies;
    \item All other rules are not activated.
\end{itemize}

After all rule applications and orientation propagation, the resulting PAAGG (Figure \ref{fig:final} is maximally informative: each remaining $\circ$ mark reflects a true ambiguity in the equivalence class $O\text{-Equiv}(D_O)$.

\begin{figure}[h!]
    \centering
    \includegraphics[width=\textwidth]{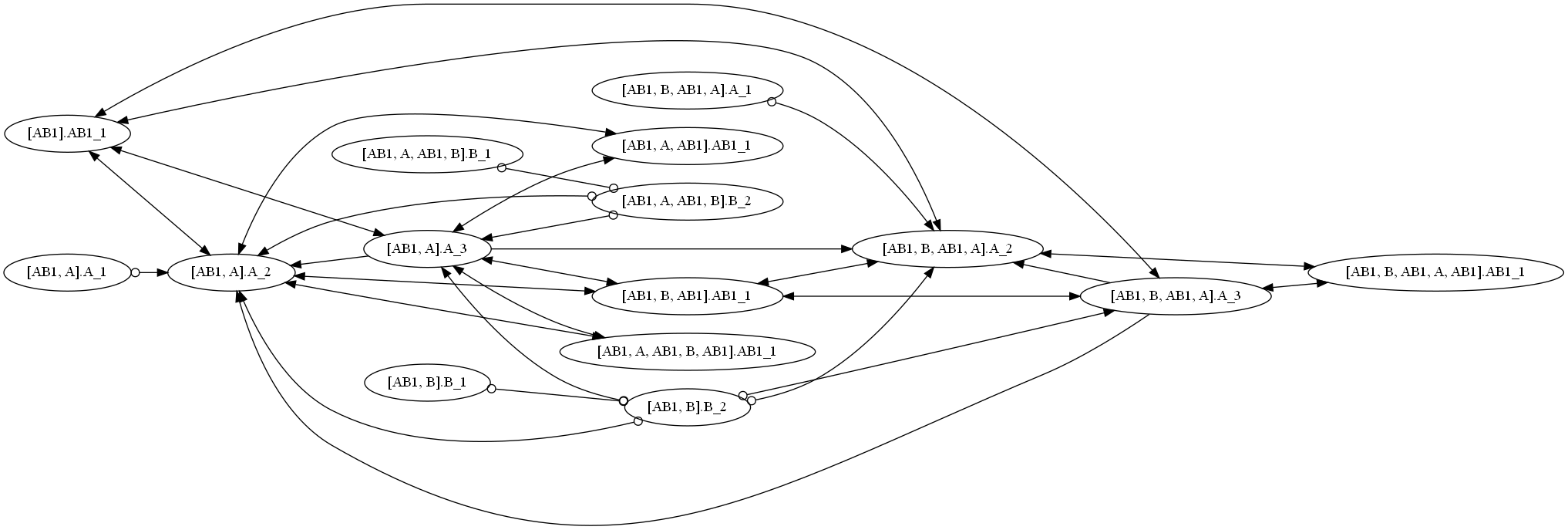}
    \caption{Final PAAGG with maximally informative edge orientations.}
    \label{fig:final}
\end{figure}

\subsection*{Output – Extraction of Dependencies}

From the oriented PAAGGs, the algorithm extracts the required and possible underlying dependencies. These define the Partial Ancestral Relational Model, shown in Figure \ref{fig:parm}. 

\begin{figure}[H]
    \centering
    \includegraphics[width=0.8\textwidth]{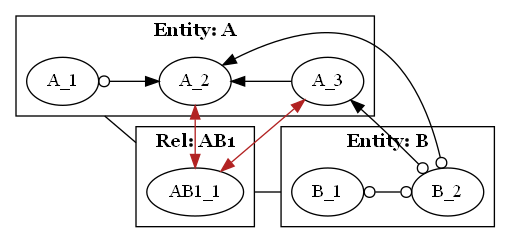}
    \caption{Learned PARM for the example model.}
    \label{fig:parm}
\end{figure}

\section{Proofs}\label{proofs}

This section contains complete proofs for all the theoretical results presented in the main paper.

\begin{lemma}\label{lemma:gg}
    Given a relational causal model structure $\mathcal{M}$ and perspective $\mathcal{B}$, if an abstract ground graph $AGG_{\mathcal{M}\mathcal{B}}$ is ancestral, then all ground graphs $GG_{\mathcal{M}\sigma}$, with skeleton $\sigma\in\sum_\mathcal{S}$, are ancestral.
\end{lemma}
\begin{proof}
    From the definition of \citet{ZHANG20081873}, a graph is ancestral if:
    \begin{enumerate}
    \item There is no directed cycle, i.e., B$\rightarrow$A is in $G$ and A is an ancestor of B (meaning there's a directed path from A to B);
    \item There is no almost directed cycle, i.e., B$\leftrightarrow$A is in $G$ and A is an ancestor of B;
    \item For any undirected edge A\textemdash B, both A and B have no parent or spouses, i.e., X, Y such that either or both A$\leftrightarrow$X or B$\leftrightarrow$Y.
    \end{enumerate}
For each of the three conditions, we must demonstrate that if the AGG is ancestral, all GGs must likewise be ancestral to prove this lemma. Given the definition of the abstract ground graph building process in Definition 5.2 and Theorem 5.2 from \citet{maier2014reasoning}, we know that the AGG is sound and complete for all ground graphs for a given perspective and hop threshold $h$. This suggests that the AGG captures every dependent path between two variables in every GG. In the same way, each path of dependence between two variables in the AGG is mirrored in at least one GG. We now verify the lemma for the three conditions of ancestrality:
    \begin{enumerate}
        \item Assume that the AGG is ancestral and that one of the ground graphs, $G$, has a directed cycle between $A$ and $B$ to provide a contradiction. Consequently, the two dependence paths in $G$ will also be present in the AGG, resulting in a directed cycle. Thus, the maximal ancestral abstract ground can't be ancestral;
        \item Similar reasoning can be carried when considering almost directed cycles containing double-arrowed edges (in the case of \textit{Maximal Ancestral Abstract Ground Graphs}), thus verifying the lemma for this condition as well;
        \item Given the assumptions of the underlying structure's acyclicity and no selection bias (i.e., no variables are in the set \textbf{S}), an undirected edge cannot exist as it corresponds to the presence of selection variables, of which $X$ and $Y$ are the cause \cite{ZHANG20081873}. Thus, this condition does not apply to AGGs.
    \end{enumerate}
\end{proof}
Lemma \ref{lemma:gg} guarantees that the theoretical reasoning devised for MAGs and PAGs can also be applied to the relational counterparts we provide in this work, MAAGGs, and PAAGGs. In other words, we know that the ancestrality of these relational lifted representations corresponds to the same ancestrality properties in the underlying ground graphs and, thus, in the underlying latent causal relational causal model we want to learn.
\begin{proposition}
    Given a relational causal model $\mathcal{M}_\textbf{L}(\mathcal{S},\mathcal{D})$ with hop threshold $h$, and its respective latent abstract ground graph $LAGG$:
    \begin{enumerate}[label=\Roman*.]
    \item The constructed MAAGG probabilistically and causally represents $LAGG$ and thus the underlying relational causal model; \label{item1} 
    \item Assuming a sound and complete procedure to construct the $PAAGG$, it correctly represents the Markov equivalence class of the produced $MAAGG$ and, therefore, of $LAGG$ and the underlying model $\mathcal{M}_\textbf{L}$.
    \end{enumerate}
\end{proposition}
\begin{proof} 
\begin{enumerate}[label=\Roman*.]
\item We can demonstrate that the MAAGG, constructed from $LAGG$ by employing the same MAG construction procedure provided in \citet{ZHANG20081873}, probabilistically and causally represents it as a result of theorem 4.18 of \citet{Richardson2002AncestralGM}, where they show that the independence model corresponding to the constructed graph coincides with the one obtained by marginalizing and conditioning the model on the original graph ($LAGG$). Furthermore, the MAAGG also represents the model $\mathbfcal{M}_\mathbf{L}$, which follows from Lemma \ref{lemma:gg}.
\item Under the assumption of a sound and complete procedure for generating said representation (i.e., the RelFCI algorithm), the PAAGG represents the Markov equivalence class containing the MAAGG. This proof follows from \citet{ZHANG20081873}: the PAAGG, constructed from a sound and complete algorithm that outputs a set of graphs which includes all the causal relationships consistent across all MAAGGs, accurately represents the equivalence class. This is because it captures the uncertainty (circle marks) where the data does not provide enough information to distinguish between different causal structures. Finally, from \ref{item1}, we can prove that the PAAGG also represents the equivalence class of $LAAG$ and the underlying model $\mathbfcal{M}_\mathbf{L}$.
\end{enumerate}
\end{proof}

\begin{proposition} \label{prop:hop}
Given a latent relational causal model $\mathcal{M}_\textbf{L}(\mathcal{S},\mathcal{D})$ with hop threshold $h$ and its corresponding PARM $\mathbfcal{M}$, the hop threshold $h'$ of the $PAAGG_{\mathbfcal{M}\mathcal{B}}$ for any perspective $\mathcal{B}$ can be at most $2h$.
\end{proposition}
\begin{proof}
Let us consider a scenario within a relational causal model that allows relational latent variables to be observed and in which the non-dependence of these variables holds (i.e., no latent variable causes another latent variable, which entails there cannot exist a chain of dependencies consisting of multiple consecutive latent variables). 
For the sake of clarity, we will focus on three entities, A, B, and C, each containing one attribute, respectively $A1$, $B1$, and $C1$, with $B1$ designated as latent as in Figure \ref{fig:hop}. 
Suppose we were to connect them, using B1 as the connecting bridge between the other two attributes using the following dependencies: $[A, B].B1 \rightarrow [A].A1$ and $[C, B].B1 \rightarrow [C].C1$, both of which require a hop threshold of one to be represented.
After removing the assumption of having all variables observed, the scenario reverts to one where $B1$ is latent, which means that the dependencies between $B1-A1$ and $B1-C1$ are no longer observable. The possible existing dependencies, containing only relational variables with a path of length two (hop threshold equal to one), make the model unable to express the dependencies among the attributes of different entities, e.g., $[A,B,C].C1\rightarrow[A].A1$ and $[C,B,A].A1\rightarrow[C].C1$.
To account for the relational dependencies between the two entities, we need a relational path that is long enough to traverse the entities and describe the relationship between the variables expressed by the model, which requires twice the original hop threshold of one.
\end{proof}
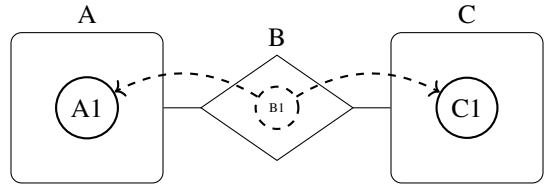
\begin{figure}[H]
    \centering
    \begin{tikzpicture}
        \node[draw, rounded corners, rectangle, minimum width=2cm, minimum height=2cm] (A) at (0, 0) {};
        \node[above] at (A.north) {A};
    
        \node[draw, rounded corners, rectangle, minimum width=2cm, minimum height=2cm] (B) at (5, 0) {};
        \node[above] at (B.north) {C};
    
        \coordinate (N) at (2.5, 0.7);
        \coordinate (E) at (3.5, 0);
        \coordinate (S) at (2.5, -0.7);
        \coordinate (W) at (1.5, 0);
        \node[above] at (N.north) {B};
    
        \node[draw, circle, thick] (A1) at ([xshift=-0.0cm, yshift=-0.0cm] A.center) {A1};
    
        \node[draw, circle, font=\tiny, thick, dashed] (B1) at ([yshift=-0.7cm] N.center) {B1};
        
        \draw (N) -- (E) -- (S) -- (W) -- cycle;
    
        \node[draw, circle, thick]  (C1) at ([xshift=-0.0cm, yshift=-0.0cm] B.center) {C1};
      
        \draw (A) -- (W);
        \draw (E) -- (B);
    
        \path (B1) edge[bend right, thick, ->, dashed] (A1);
        \path (B1) edge[bend left, thick, ->, dashed] (C1);
    \end{tikzpicture}
    \caption{Example of Relational Causal Model with a latent variable}
    \label{fig:hop}
\end{figure}

\begin{theorem}
    Let G be the partially oriented PAAGG from perspective B with the correct set of adjacencies, unshielded colliders oriented correctly through CD and RBO, and as many edges as possible oriented through KNC, CA, MR3, and the purely common cause of RBO. Then, the rules R4-R10 from FCI and the orientation propagations are sound.
\end{theorem}
\begin{proof}
Given lemma \ref{lemma:gg}, the proof derives from \citet{spirtes1995causal} and \citet{ZHANG20081873}. A rule is sound if the arrows and tails used in the resulting PAAGG are invariant. Therefore, we need to prove that any mixed abstract ground graph $G$ that violates a rule does not belong to the equivalence class \textit{O-Equiv}($\mathcal{D}_{\textbf{O}}$), that is, it is not ancestral or Markov equivalent to the original MAAGG.
The proof for rule R4 is identical to the proof by induction provided in \citet{spirtes1995causal}, stating that by applying iteratively rule R4 on a PAAGG $G$ oriented using rules CD, CA, KNC, and MR3, the resulting graph $G_i$ at each iteration $i$ maintains its ancestral properties for the equivalence class \textit{O-Equiv}($\mathcal{D}_{\textbf{O}}$). The proof for the remaining rules is taken from \citet{ZHANG20081873}:
\begin{itemize}
\item R5: The rule states that the path $p = \langle\alpha, \gamma,...,\theta,\beta,\alpha\rangle$ consists of an uncovered cycles of only circle marks. If we assume instead that a graph $G$ has an arrowhead on this cycle because of KNC, this cycle must be directed to avoid unshielded colliders. But by doing so, the graph is not ancestral;
\item R6: Any graph $G$ that contains the opposite orientation than the one stated by the rule, i.e., $\alpha$\textemdash$\beta\leftarrow\ast\gamma$, is not ancestral;
\item R7: Supposed that a graph $G$ has an arrowhead into $\beta$ as opposed to the rule. Therefore, the triple can be oriented as $\alpha$\textemdash$\beta\leftarrow\ast\gamma$ or $\alpha\rightarrow\beta\leftarrow\ast\gamma$. In the former case, $G$ is not ancestral. In the latter, it contains an unshielded collider not present in the original MAAGG;
\item R8: If a graph $G$ instead of $\alpha\rightarrow\gamma$ contains $\alpha\leftrightarrow\gamma$, then there is an almost directed cycle or an arrowhead into an undirected edge. In both cases, the graph is not ancestral;
\item R9: The same proof for R5 can be applied for this rule;
\item R10: The rule states that $\langle\mu,\alpha,\omega,\rangle$ is not a collider in the original MAAGG. Assume that a graph $G$ in the equivalence class contains $\alpha\leftrightarrow\gamma$ instead of the rule specification. Then, for $G$ to be ancestral, one or more edges out of $\alpha$ must be directed. Therefore, to avoid unshielded colliders not in the original MAAGG, $p_1$ or $p_2$ must be a directed path, making $alpha$ an ancestor of $gamma$ and thus $G$ not ancestral. 
\end{itemize}
Finally, considering that the rules are proven sound and, as such, all orientations produced are correct, it is straightforward to prove that the respective orientation propagation procedure is sound, following from \cite{maier2013sound}.
\end{proof}
The following two lemmas for the arrowhead and tail completeness make use of a representation defined as \textit{chordal graph}, established in \citet{meek1995causal} and extended in \citet{maier2013sound} for relational data. This representation is an undirected graph where every undirected cycle of length four or more has an edge between two nonconsecutive vertices on the cycle. In chordal graphs, a total order $\alpha$ is consistent with respect to $AGG$ if and only if $AGG_\alpha$ (abstract ground graph in which $A\rightarrow B$ if and only if $A<B$ with respect to $\alpha$) has no unshielded colliders. Furthermore, for all adjacent vertices $A$ and $B$, there exists consistent total orderings $\alpha$ and $\gamma$ such that $A\leftarrow B\in AGG_\alpha$ and $A\rightarrow B\in AGG_\gamma$. 

\begin{lemma}
Let G be a partially oriented PAAGG with correct adjacencies. Then, exhaustively applying CD, RBO, KNC, CA, MR3, and R4, all with orientation propagation of edges, produces a PAAGG G' in which for every circle mark there exists a MAAGG in the \textit{O-Equiv}($\mathcal{D}_{\textbf{O}}$) class with a corresponding tail mark.
\end{lemma}

\begin{proof}  
The proof follows from Theorem 4.3 of \citet{ali2012towards}. They prove arrowhead completeness for a different graph representation for the Markov equivalence class of MAGs, \textit{Joined Graphs}, which do not distinguish between tail marks and circle marks, provided that the work focused explicitly on arrowhead edge orientations. The same reasoning can be used to ancestral graphs and, with Lemma \ref{lemma:gg}, to PAAGGs. Let $G'$ be the PAAGG with as many edges orientated using CD, RBO, CA, MR3, and R4. For these proofs, we define the edge marker $\otimes$, which corresponds to either a circle or edge mark. There are four steps to prove the arrowhead completeness: 
\begin{enumerate}
    \item Removing any non-directed edge in $G'$ creates a disjoint union of maximal ancestral PAAGGs. Assume for contradiction that the graph $G^*$ obtained by removing undirected edges is not ancestral. Given that $G^*$ does not contain undirected edges, it cannot contain the following configurations: $A\otimes\rightarrow B$\textemdash$C$ or $A\ast\rightarrow B$\textemdash$C$\textemdash$D\rightarrow A$. Therefore, it contains a partially directed k-cycle such as $X\ast\rightarrow Y\rightarrow ... \rightarrow Z\rightarrow X$. It can be easily proven that no such cycle can exist without contradiction for $k\geq 3$; therefore, $G^*$ is both ancestral and maximal (Lemma 4.1 of their work that proves that the oriented $G'$ contains only triangles with the following forms: \\
    (i) $B\rightarrowast A\leftarrowast C \ast$ \textemdash $\ast B$; (ii) $B \ast$ \textemdash $A$ \textemdash $\ast C \ast$ \textemdash $\ast B$; or (iii) $Y\rightarrowast A $\textemdash $\ast C \leftarrowast B$).
    \item No replacement of the undirected edges in $G'$ by directed edges will result in non-ancestral structures such as partially directed cycles, unshielded colliders, colliders with order, or inducing paths with non-adjacent endpoints that include an edge oriented by the orientation rules. The absence of these non-ancestral structures is a direct consequence of Lemma 4.1. \label{item2}
    \item By removing all directed edges and undirected ones with no parents or spouses from $G'$, the resulting AGG $U$ is a disjoint union of chordal undirected graphs. Assume for contradiction that the orderings of $U$ lead to unshielded colliders. From \ref{item2}, we know that a replacement of undirected edges could generate a collider with order or inducing paths with non-adjacent endpoints. It's also possible to prove by contradiction that if $U$ is not chordal, then the subgraph $U'$ of the partially oriented PAAGG corresponding to $U$ must contain the same non-chordal properties (i.e., unshielded colliders), which is not possible as $U'$ cannot contain an unshielded collider given the orientation provided by the CD rule. Therefore, $U$ must be chordal.
    \item By definition of chordal graph, for every pair $(A,B)$ there are at least two orderings such that $A\rightarrow B$ in one and $A\leftarrow B)$ in the other. Therefore, $G'$ is maximally oriented, and as such, the rules  CD, CA, MR3, and R4 are arrowhead complete. 
\end{enumerate}
\citet{maier2013sound} demonstrates the completeness of the merely common cause rule of RBO, which establishes edge orientation through arrowhead marks only. Consider again the PAAGG $G'$. Assume by contradiction that there's an edge in $G'$ with a circle mark (without loss of generality, $A\rightarrowcircle B$), such that there are no MAAGGs in \textit{O-Equiv}($\mathcal{D}_{\textbf{O}}$) with a corresponding tail mark for that edge. This requires that the edge mark correspond to an arrowhead in both the equivalence class and the generated PAAGG. Based on the completeness proofs provided above, one of the rules would have orientated that edge mark with an arrowhead. As a result, there must be a MAAGG in \textit{O-Equiv}($\mathcal{D}_{\textbf{O}}$) that has the edge $A\rightarrow B$, also known as a tail mark. 
\end{proof}

\begin{lemma}
Let G' be the partially oriented PAAGG with correct adjacencies and unshielded colliders, and as many edges oriented with KNC, CA, and MR3, all with orientation propagation. Then, applying rules R5-R10, together with orientation propagation, produces a PAAGG G'' such that for every circle mark, there exists a MAAGG in \textit{O-Equiv}($\mathcal{D}_{\textbf{O}}$) in which the corresponding mark is an arrowhead.
\end{lemma}
\begin{proof}
Using Lemma \ref{lemma:gg}, we may follow \citet{ZHANG20081873} tail completeness proof. We show that any PAAGG edge with a ◦ mark (e.g., ◦\textemdash, ◦\textemdash◦, ◦$\rightarrow$) corresponds to an arrowhead in a MAAGG in the equivalence class. 
For the first two types of edges (◦\textemdash, ◦\textemdash◦), we make use of some properties of PAGs, proven in \citet{ZHANG20081873} and adapted to PAAGGs:
\begin{enumerate}[label=\textbf{P}\arabic*]
    \item Given a triple A, B, C in a PAAGG, if $A\rightarrowast B\circleast C$, then there is an edge $A\rightarrowast C$. In addition, if $A\rightarrow B$, then the edge between A and C cannot be $A\leftrightarrow C$;
    \item Given two vertices, A and B, in a PAAGG, if $A$\textemdash◦$B$, then there is no edge into A or B;
    \item Given a triple A, B, C in a PAAGG, if $A$\textemdash◦$B\circleast C$, then there is an edge between A and C. Furthermore, if $A$\textemdash◦$B$◦\textemdash◦$C$, then the edge between A and C is $A$\textemdash◦$C$; if $A$\textemdash◦$B\rightarrowast C$, then either $A\rightarrow C$ or $A\rightarrowast C$;
    \item Given two vertices, A and B, in a PAAGG, if $A$\textemdash◦$B$, then there is no cycle with the following structure $A$\textemdash◦$B$\textemdash◦$...$\textemdash◦$A$.
\end{enumerate}
With these properties, it can be proven that:
\begin{itemize}
    \item For every edge $A$◦\textemdash◦$B$ in the subgraph obtained by keeping only ◦\textemdash◦ edges from the PAAGG (which we denote as $P^C_{AAGG}$), the subgraph can be oriented into two DAGs without unshielded colliders such that $A\rightarrow B$ in one and $A\leftarrow B$ in the other. This is proven by showing that $P^C_{AAGG}$ is chordal: assume by contradiction that there is a non-chordal cycle $\langle X, Y, W, ..., Z \rangle$. This implies that any non-consecutive vertices in the cycle are not adjacent in either $P^C_{AAGG}$ or the original PAAGG, as otherwise they would be connected by a ◦\textemdash◦ edge (deriving from \textbf{P}1 and \textbf{P}3) and as such connected in the $P^C_{AAGG}$ as well. Therefore, this non-chordal cycle also appears in the PAAGG, which should have been oriented with rule R5. Therefore the $P^C_{AAGG}$ is chordal.
    \item Let $H$ be the graph obtained from the following steps applied to the PAAGG:
    \begin{enumerate}
        \item orient all $\rightarrowcircle$ and \textemdash◦ edges into directed ones, i.e., $\rightarrow$;
        \item orient the $P^C_{AAGG}$ into a DAG with no unshielded collider.
    \end{enumerate}
    Then $H$ belongs to the equivalence class represented by the PAAGG: \\
    \textbf{P}1-4 ensure that no directed or almost directed cycle is generated after the first step. For step 2, \textbf{P}1 and \textbf{P}3 ensure that in the $P^C_{AAGG}$ no new directed or almost directed cycles will be generated in $H$, and furthermore, no new edge into any vertex incident to undirected edges and no inducing paths between any non-adjacent vertices appear. This verifies that $H$ is ancestral and maximal. It is easy to prove then that $H$ belongs to the equivalence class as \textbf{P}1-3 guarantee that no new unshielded colliders are created, and as no new bi-directed edges are created also the discriminating path condition for Markov equivalence between $H$ and the PAAGG is verified.
\end{itemize}
These two theoretical conclusions guarantee that no circle on a PAAGG's ◦\textemdash and ◦\textemdash◦ edges corresponds to an invariant tail. The proof for the ◦$\rightarrow$ edge comes from Theorem 3 in \citet{ZHANG20081873}, which uses the chordal graph representation established in \citet{meek1995causal} and extended in \citet{maier2013sound} for relational data. 
For the PAAGG $ G''$, a proof by contradiction similar to the one provided in Lemma 2 can be carried out for every circle mark corresponding to an arrowhead in at least one MAAGG in the equivalence class \textit{O-Equiv}($\mathcal{D}_{\textbf{O}}$).
\end{proof}

\setcounter{theorem}{2}
\begin{theorem}
Given a schema and a probability distribution P(\textbf{V}) with $\textbf{V}=\textbf{O}\cup\textbf{L}\cup\textbf{S}$, the output of RelFCI is a correct maximally informative PAAGG, and thus a maximally informative PARM $\mathbfcal{M}$, assuming perfect conditional independence tests and sufficient hop threshold $h'$.
\end{theorem}
\begin{proof}
The following proof sketch is adapted from \citet{maier2014reasoning}. Given a sufficient $h'$ at least equal to $2h$ (Proposition \ref{prop:hop}), the set of potential dependencies $PDs$ includes all true dependencies that generate the respective $MAAGG$, which implies the generation of the correct adjacencies, which include the true causes for each relational variable. The unoriented PAAGGs are then constructed using the procedure from \citet{maier2014reasoning}. Assuming perfect conditional independence tests, the algorithm maintains only the correct edges for the PAAGGs. $S$ and $U$ also contain the correct separating sets for every pair of nonadjacent variables and the true unshielded colliders. Next, RelFCI orients all unshielded colliders using either CD or RBO and then, finally, produces a maximally informative PAAGG $G$ and PARM $\mathbfcal{M}$ as an implication of Theorem 1 and Theorem 2.
\end{proof}

\section{Additional Results}\label{res}

We further evaluated the performance of RelFCI in the absence of latent variables to establish a fair comparison with RCD under causal sufficiency. The experimental setup mirrors that described in Section \ref{sec:setup}, and the results are presented in Figure \ref{fig:no-lat}. As shown, RelFCI achieves precision and recall comparable to, and in some configurations slightly exceeding, those of RCD. These results demonstrate that RelFCI maintains high accuracy even when latent confounders are not present, confirming its soundness in recovering the true causal structure in standard relational settings.
\begin{figure}[ht]
    \centering
    \includegraphics[width=\textwidth]{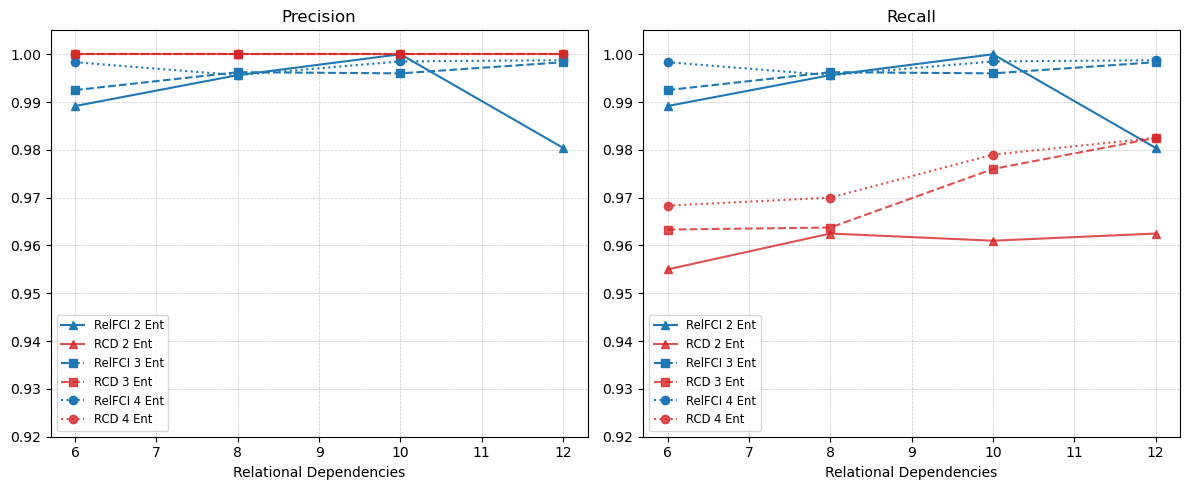}
    \caption{RelFCI Precision and Recall performance with no latent variables.}
    \label{fig:no-lat}
\end{figure}

\end{appendix}

\end{document}